\documentclass[10pt,journal,compsoc]{IEEEtran}
\usepackage[numbers,sort&compress]{natbib}
\usepackage[hyphens]{url}
\usepackage{graphicx}
\usepackage{booktabs}
\usepackage{amsmath,amssymb,amsthm}
\usepackage{microtype}
\makeatletter
\def\@IEEEsectpunct{\ \,}%
\renewcommand{\paragraph}{\@startsection{paragraph}{4}{\z@}%
  {1.5ex plus 0.5ex minus .2ex}{-1em}{\normalfont\normalsize\bfseries}}
\makeatother
\theoremstyle{plain}\newtheorem{theorem}{Theorem}\newtheorem{proposition}{Proposition}
\newtheorem{lemma}{Lemma}\newtheorem{corollary}{Corollary}
\theoremstyle{definition}\newtheorem{definition}{Definition}
\theoremstyle{definition}\newtheorem{construction}{Construction}
\theoremstyle{remark}\newtheorem{remark}{Remark}
\newcommand{\Ch}{\mathrm{Ch}}
\newcommand{\E}{\mathbb{E}}\newcommand{\Prob}{\mathbb{P}}\newcommand{\R}{\mathbb{R}}
\newcommand{\one}{\mathbf{1}}\newcommand{\AUC}{\mathrm{AUC}}\newcommand{\MI}{\mathrm{MI}}
\newcommand{\Var}{\mathrm{Var}}\newcommand{\KL}{\mathrm{KL}}\newcommand{\TV}{\mathrm{TV}}
\DeclareMathOperator*{\argmax}{arg\,max}

\begin{document}

\title{RouteGuard: Certifying Routing Gain in LLM Multi-Agent Systems\\
When Complementarity Is Not Enough}

\author{Anchen~Sun and Kaiqi~Yang%
\IEEEcompsocitemizethanks{%
\IEEEcompsocthanksitem A.~Sun is with the Department of Electrical and
Computer Engineering, University of Miami, Coral Gables, FL 33146 USA, and is
now with Google. This work was initiated while a Ph.D.\ candidate at the
University of Miami.\protect\\
E-mail: anchensun@google.com
\IEEEcompsocthanksitem K.~Yang is with the Department of Mathematics,
University of Miami, Coral Gables, FL 33146 USA.\protect\\
E-mail: yangkaiqi2020@gmail.com}}

\markboth{Preprint --- under review}%
{Sun \MakeLowercase{\textit{and}} Yang: RouteGuard: Certifying Routing Gain
in LLM Multi-Agent Systems}

\IEEEtitleabstractindextext{%
\begin{abstract}
Multi-agent LLM systems route among model-backed advisors, yet a deployer
rarely knows before shipping whether routing will help at all. Prevailing
routers optimize a gate's AUC and presume that advisor complementarity
suffices. We show that neither determines the deployable gain. We introduce
\textbf{RouteGuard}, a deployment-certification framework. Routing gain
decomposes as $G{=}\pi\,\Delta_E$, and the achievable gain is governed by a
conditional-regret functional $\Phi$, not by AUC. A finite-sample
certification bracket comes with a matching Le~Cam lower bound,
constant-sharp over the fixed-activity class, and a robustness phase
transition. On two benchmarks the framework
acts as a guardrail. On RouterBench ($11$ cross-family models) the verdict
depends on the sampling unit: the protocol \emph{certifies} a gain over
GPT-4 under prompt-level sampling and \emph{withholds} it under
workload-cluster resampling, because the gain rests on $3$ of $86$ workload
cells. On OpenRCA (three Gemini advisors) the advisors are statistically
\emph{redundant}: the realized oracle sits at or below the independence
baseline in all pools we tested ($221$ RouterBench pools and three OpenRCA
distributions), so the protocol correctly \emph{refuses} to certify. A
pre-registered semi-synthetic control confirms calibration: the protocol
certifies a genuine gain once $m\ge m^\star$ and does not certify a true
null.
Code and frozen artifacts will be released with the published version.
\end{abstract}

}

\maketitle
\IEEEdisplaynontitleabstractindextext
\IEEEpeerreviewmaketitle

\IEEEraisesectionheading{\section{Introduction}\label{sec:intro}}

\IEEEPARstart{L}{arge-language-model} (LLM) multi-agent systems
increasingly route among
model-backed advisors: a gate reads side information and commits the system to
one advisor's answer. If different advisors fail on different inputs, a router
should beat the primary (strongest single advisor), the intuition driving a fast-growing
routing literature (e.g., MasRouter~\citep{masrouter2025}) and work on
diagnosing multi-agent failures (Who\&When~\citep{whoandwhen2025};
MAST~\citep{mast2025}). We address what that literature leaves open: \emph{when
does selecting among imperfect agents actually help, and how would a deployer
know before shipping?} Small-sample routing gains are fragile, and a gate that
looks predictive in-sample can be uninformative about which advisor is right. A
deployer therefore needs a \emph{certification protocol}: from $m$ evaluation
examples, either certify routing---report a lower bound on the deployed gain
that holds with probability at least $1-\delta$---or refuse to endorse it. The
theory behind such a protocol must also identify \emph{which} property of the
gating signal determines whether any gain exists. This paper supplies both and
stress-tests them on two benchmarks.

\paragraph{Thesis: gain is governed by gating informativeness.}
Whether routing can help at all is governed by a single scalar property of
the gating signal: how much the signal reveals about \emph{which} advisor is
correct on the current instance. We call this quantity the gate's
\emph{informativeness}, write it $\Phi$, and define it formally in
\S\ref{sec:theory}. The best achievable routing gain equals $\Phi$ exactly
(Theorem~\ref{thm:theory-T1}). The gate's AUC does not determine $\Phi$.
Advisor complementarity alone does not guarantee $\Phi>0$ either. The certification bracket of \S\ref{sec:theory} is built
on this objective.

\paragraph{Empirical findings in brief.}
On OpenRCA~Bank~\citep{openrca2025} (three Gemini advisors, $135$ incidents)
the router realizes
$+0.0$\,pp despite a $+9.6$\,pp oracle headroom. The advisors co-fail more
often than independent advisors of the same accuracies would, and the gate is
uninformative, so the protocol refuses. A pre-registered semi-synthetic
control that the protocol \emph{does} certify shows the refusal reflects
calibration rather than incapacity (\S\ref{sec:empirical}).

\paragraph{Proof status.}
T1 (design objective) is fully proven. T2 (scope) is settled within its
structural vocabulary: complementarity is proven necessary, our own
conjectured completion is refuted, and no Boolean combination of the
structural conditions decides positive gain (Rem.~\ref{rem:theory-T2});
the operative quantity is $I_{\mathrm{TV}}$. The bracket
constant is sharp at leading order---an asymptotic statement, reconciled with
the strict finite-$m$ conservatism of the bracket in \S\ref{sec:theory}---and
the class-level minimax constant is sharp
($c^\star_{\mathrm{cert}}{=}2V^\star$; Appendix~\ref{app:sharpconst}).

\subsection*{Contributions}

\begin{enumerate}
  \item \textbf{Decomposition and design objective (proven).}
        Routing gain factors exactly into how often the router intervenes
        and how well it does when it intervenes, and the largest achievable
        gain is the gate's informativeness. Formally, $G(R)=\pi\,\Delta_E$
        for any router (Prop.~\ref{prop:theory-decomp}), and
        $\max_RG(R)=\Phi=\E[\max_j\eta_j(T)-\eta_1(T)]$, attained by the
        Bayes selector (Theorem~\ref{thm:theory-T1}). The gate's AUC is not
        a sufficient statistic for~$\Phi$: two laws with identical
        $\AUC{=}\tfrac12$ yield $\Phi{=}0$ vs $\tfrac14$.

  \item \textbf{Finite-sample certification bracket with Le~Cam lower bound.}
        The certificate returns a high-confidence lower bound on the
        deployed gain or refuses, and no test can need substantially fewer
        samples. Formally: a Bernstein bound $B(m,\delta)$ sharp at leading
        order, a Le~Cam lower bound matching in scaling---and, over the
        fixed-activity class, in the exact minimax constant
        ($c^\star_{\mathrm{cert}}{=}2V^\star{=}2\pi$, attained by the
        variance-capped certificate)---and a robustness phase transition at
        $\rho^\star{=}\pi\Delta_E/2$. On the audited anchor the fixed-$M{=}2$
        bracket needs $m^\star{=}312$.

  \item \textbf{Independence baseline and the redundancy pattern.}
        Advisor pools that look diverse on paper can be effectively
        redundant, and redundancy caps what any router can recover. The
        excess over independence $\mathcal{E}$ enters the routing ceiling
        additively ($\Phi\le H_{\mathrm{ind}}+\mathcal{E}$,
        Lemma~\ref{lem:ceiling}). On all three OpenRCA distributions and
        every one of $221$ RouterBench pools, $\mathcal{E}\le0$: the
        advisors co-fail more often than independent advisors of the same
        accuracies would, which lowers the attainable ceiling well below
        what the marginals suggest. As a scope result, complementarity (C1)
        is necessary for $G>0$ but not sufficient. The operative structural
        quantity is $I_{\mathrm{TV}}$ (Rem.~\ref{rem:theory-T2}).

  \item \textbf{End-to-end guardrail validation on two benchmarks.}
        On RouterBench ($11$ cross-family models) the verdict depends on
        the sampling unit: the protocol \emph{certifies} a gain over GPT-4
        under prompt-level sampling and correctly \emph{withholds} it under
        workload-cluster resampling, because the gain rests on $3/86$
        cells. On OpenRCA~Bank ($N{=}3$ advisors, $m{=}135$ incidents) the
        gate is uninformative and the protocol refuses. A pre-registered
        semi-synthetic control certifies a true gain at $m\ge m^\star$,
        withholds below~$m^\star$, and does not certify a true null.
\end{enumerate}

\section{Related Work}
\label{sec:related}

Our contribution sits at the intersection of failure \emph{attribution} in
multi-agent systems, model \emph{routing} for LLMs, and finite-sample
\emph{certification}. Our question belongs to the latter two: whether and when
a routing gain can be statistically guaranteed at deployment. From the first
we borrow only diagnostic vocabulary.

\subsection{Failure attribution in multi-agent systems}
A fast-growing cluster localizes \emph{which} agent or step caused a failure:
\emph{Who\&When}~\citep{whoandwhen2025} benchmarks automated attribution,
\emph{MAST}~\citep{mast2025} taxonomizes failure modes, and
\emph{AgentDebug}~\citep{agentdebug2025}, \emph{A2P}~\citep{a2p2025}, and
\emph{AgenTracer}~\citep{agentracer2025} trace error propagation---a
\emph{post-hoc, instance-level} question, whereas ours is \emph{prospective and
population-level}: certifying that the gain stays positive under bounded
shift. Attribution can supply candidate gates, but we never require the gate
to explain \emph{why} an advisor erred.

\subsection{LLM routing and model selection}
A second cluster routes queries among models---\emph{MasRouter}~\citep{masrouter2025},
\emph{Causal LLM Routing}~\citep{causalrouting2025}, \emph{Universal Model
Routing}~\citep{universalrouting2025}, and diagnostics such as
\emph{RouterBench}~\citep{routerbench2024}, \emph{RouterXBench}~\citep{routerxbench2025},
and \emph{When Routing Collapses}~\citep{routingcollapses2025}. Our objective
differs from this literature in two ways. The prevailing objective is
cost-adjusted accuracy under a learned gate, whereas we certify the gain over
a fixed primary. The literature also treats gate--correctness AUC as the
quality target, while Theorem~\ref{thm:theory-T1} proves that the
informativeness $\Phi$ governs attainable gain. ``Routing collapses''
corresponds in our language to the regime $\pi\Delta_E\to0$, detected before
deployment.

\subsection{Learning to defer, ensembles, and cascades}
The Bayes selector over $\eta_j(t)$ is multi-expert
\emph{learning-to-defer}~\citep{madras2018,mozannar2020,verma2023}, and $\Phi$
is the value of deferral. The distinction is that the deferral literature
optimizes the \emph{policy} while we certify the deployed \emph{value} of a chosen
policy---a finite-sample accept/refuse decision it does not provide. The
selector itself is standard; the new objects are the bracket $B(m,\delta)$,
the phase transition $\rho^\star$, and the redundancy diagnostic
$\mathcal{E}$. Diversity diagnostics
($\phi$, oracle ceilings) are classical ensemble theory~\citep{kuncheva2003};
$\mathcal{E}$ reframes them as a deployment screen with a quantitative role in the
ceiling (Lemma~\ref{lem:ceiling}). Cascades~\citep{frugalgpt2023} and selective
prediction~\citep{geifman2017} optimize cost/coverage objectives. We scope to
accuracy-only, where the AUC-vs-$\Phi$ phenomenon is cleanest, and cost-aware
extensions compose with the same bracket.

\subsection{Finite-sample certification}
The closest prior work shares our finite-sample, phase-transition machinery.
\emph{Proactive-Routing}~\citep{proactiverouting2026} gives one-sided conformal
per-decision guarantees. We give a certification \emph{bracket} on the gain
functional with a Le~Cam lower bound that conformal calibration lacks.
\emph{Cer-Eval}~\citep{cereval2026} is sequential and certifies \emph{point}
selection. We are fixed-$m$ and certify \emph{advisor} selection. Certifying
$G>0$ at a small gap is a best-arm-identification instance~\citep{bestarm2025};
our additions are the routing decomposition, the phase transition
$\rho^\star=\pi\Delta_E/2$, and a bracket whose lower bound is
constant-sharp at class level.

\subsection{RCA agents and the gating signal we target}
Our empirical surface is OpenRCA~\citep{openrca2025}. \emph{Why Do AI Agents
Systematically Fail at Cloud RCA}~\citep{whyrcafail2026} documents that root-cause-analysis (RCA) agents ignore
key-performance-indicator (KPI) categories. Our $K{=}4$ KPI-family partition injects
deterministic, family-partitioned anomaly evidence as a gate independent of the
LLM chain---on Bank it is uninformative for advisor selection, so the protocol
refuses (\S\ref{sec:empirical}).

\section{A Deployment-Certification Theory for Advisor Routing}
\label{sec:theory}

The value of an advisor router is governed by a single per-instance
\emph{informativeness} functional that admits a finite-sample certification
bracket.

\subsection{Formal setup and the central identity}
\label{sec:theory-setup}

\begin{definition}[Predictors, primary, gating signal, router]
\label{def:theory-setup}
Fix $N \ge 2$ advisors with correctness indicators $C_j \in \{0,1\}$ on
instances $X\sim\mu$, where $\mu$ is the (unknown) population distribution of
evaluation instances; write $p_j := \E[C_j]$, with the \emph{primary} $j{=}1$,
$p_1=\max_jp_j$. A \emph{gating signal} $T = (T_1, \dots, T_N)$ takes values
$T_j \in \{0, \dots, K\}$; a \emph{router} is any measurable
$R : \mathcal{T}^N \to [N]$ with accuracy $A_R := \E[C_{R(T)}]$ and
\emph{routing gain} $G(R) := A_R - p_1$. Define
$\eta_j(t) := \Prob(C_j = 1 \mid T = t)$ and the oracle accuracy
$A_\star := \E[\max_j C_j]$. A \emph{joint law} $\mathcal{L} =
\mathcal{L}(C_{1:N}, T)$ is the joint distribution of the correctness vector
and the gating signal (induced by $X\sim\mu$); $G$, $\Phi$, and $\AUC$ below
are functionals of $\mathcal{L}$.
\end{definition}

\begin{definition}[ROC area]
\label{def:theory-auc}
For ordinal $W$ and binary $C$ with $\Prob(C{=}1)\in(0,1)$,
$\AUC(W;C):=\Prob(W_+>W_-)+\tfrac12\Prob(W_+{=}W_-)$ with
$W_\pm\stackrel{d}{=}(W\mid C{=}\pm)$ independent; $W$ \emph{anti-predicts} $C$
when $\AUC<\tfrac12$.
\end{definition}

The router only changes the outcome where it deviates from the primary. The
following assumption-free identity isolates that subset.

\begin{proposition}[Routing-gain decomposition]
\label{prop:theory-decomp}
For any router $R$, let $E_R := \{t : R(t) \ne 1\}$ be the \emph{route-away
event}, $\pi := \Prob(E_R)$ its mass, and
$\Delta_E := \E[C_{R(T)} - C_1 \mid E_R]$ the conditional accuracy edge on the
routed set (with the convention $\pi\Delta_E:=0$ when $\pi=0$). Then
\[
  \boxed{\,G(R) \ =\ \pi \cdot \Delta_E\,.}
\]
Moreover $G(R)$ depends on the joint law $\mathcal{L}(C_{1:N}, T)$ only through
its restriction to $E_R$.
\end{proposition}

\begin{proof}
On $E_R^c$ the router selects the primary, so $C_{R(T)} - C_1 = 0$ there; on
$E_R$, by the tower property, $\E[(C_{R(T)} - C_1)\one_{E_R}] = \pi\Delta_E$.
Since $A_R - p_1 = \E[C_{R(T)} - C_1]$ and the integrand vanishes off $E_R$, the
two claims follow.
\end{proof}

Proposition~\ref{prop:theory-decomp} is the organizing identity: gain factors
into a \emph{quantity} ($\pi$) and a \emph{quality} ($\Delta_E$) of intervention,
making precise the failure mode our empirics exhibit: non-trivial route-away
mass with $\Delta_E=0$ (recoveries offsetting destructions), as observed on Bank.

\subsection{The design objective: informativeness, not AUC}
\label{sec:theory-T1}

We first identify the largest gain achievable by \emph{any} router. Part~(a)
below holds \emph{law-wise}---for every joint law $\mathcal{L}$, with no class
restriction; the nuisance class
$\mathfrak{L}_{p_1, A_\star} := \{ \mathcal{L} : \E[C_1] = p_1,\ \E[\max_j C_j]
= A_\star \}$ (primary strength and oracle ceiling held fixed) is needed only for
the comparison in part~(b).
Define the \emph{conditional-regret functional}
\[
\begin{aligned}
  \Phi(\mathcal{L}) &\ :=\ \E\!\big[\, \max_j \eta_j(T) - \eta_1(T) \,\big] \\
  &\ =\ \sum_t \mu_T(t)\big( \max_j \eta_j(t) - \eta_1(t) \big) .
\end{aligned}
\]
We refer to $\Phi$ as the gating signal's \emph{informativeness}: it is the
expected per-instance advantage of the best posterior-correct advisor over the
primary, and it is zero precisely when $T$ never reveals a context in which some
advisor strictly dominates the primary.

\begin{theorem}[Design objective]
\label{thm:theory-T1}
\begin{itemize}
  \item[\textnormal{(a)}] \textbf{Attainable maximum (law-wise).}
        For \emph{every} joint law $\mathcal{L}$,
        $\displaystyle \max_R G(R) = \Phi(\mathcal{L})$ exactly, attained by the
        Bayes selector $R^\star(t) := \argmax_j \eta_j(t)$.
  \item[\textnormal{(b)}] \textbf{AUC-insufficiency.} On the nuisance class
        $\mathfrak{L}_{p_1, A_\star}$ there exist
        $\mathcal{L}_a, \mathcal{L}_b \in \mathfrak{L}_{p_1, A_\star}$ with
        $\AUC(T_1; C_1)(\mathcal{L}_a) = \AUC(T_1; C_1)(\mathcal{L}_b)$ but
        $\Phi(\mathcal{L}_a) \ne \Phi(\mathcal{L}_b)$. Hence $\AUC(T_1; C_1)$ is
        \emph{not} a sufficient statistic for $\Phi$.
\end{itemize}
\end{theorem}

\begin{proof}[Proof sketch]
(a) $A_R=\E[\eta_{R(T)}(T)]\le\E[\max_j\eta_j(T)]$, with equality at $R^\star$.
(b) Two $N{=}2$ witnesses with identical $\AUC{=}\tfrac12$ realize $\Phi{=}0$
and $\tfrac14$; full constructions in Appendix~\ref{app:T1}.
\end{proof}

Theorem~\ref{thm:theory-T1} is proven. (Every theorem in the paper is also
re-verified numerically against frozen golden files. Appendix~\ref{app:veriftable} consolidates
all checks in one table.) The correct design objective is therefore a held-out
estimate of $\Phi$, equivalently of $\Delta_E$ on the route-away rows. The
empirical $\AUC(T;C)$ is not a substitute. When $T$ is continuous or learned,
estimate $\Phi$ by binning or a cross-fitted plug-in, as in \S\ref{sec:bank135}.

\subsection{The informativeness ceiling: a capturable-gain inequality}
\label{sec:theory-itv}

Theorem~\ref{thm:theory-T1} shows AUC is the wrong statistic. The following
identifies the quantity that bounds achievable gain. Write
$I_{\mathrm{TV}}(T) := \E\|\eta(T) - p\|_1$ for the total-variation
informativeness of the gate, with $\eta(t) = (\eta_1(t), \dots, \eta_N(t))$ and
$p = (p_1, \dots, p_N)$ (the expectation averages over $T$).

\begin{theorem}[Informativeness ceiling on capturable gain]
\label{thm:theory-itv}
With the primary indexed so that $p_1 = \max_j p_j$,
\[
  \Phi \;=\; \max_R G(R) \;\le\; \tfrac12\, I_{\mathrm{TV}}(T) ,
\]
with equality exactly when $\max_j \eta_j(t) = p_1 + \sum_j(\eta_j(t)-p_j)_+$
for a.e.\ $t$; the conditions \textnormal{(i)} $\eta_1 \equiv p_1$ a.e.,
\textnormal{(ii)} at most one active competitor per $t$, and \textnormal{(iii)}
active competitors mean-tied to the primary are \emph{sufficient} and canonical
(full characterization and a mutual-information companion in
Appendix~\ref{app:itv}).
\end{theorem}

\begin{proof}[Sketch]
Bound $\max_j\eta_j$ by the active competitor and apply the mean-zero identity
$\E[W_+]=\tfrac12\E|W|$; full proof in
Appendix~\ref{app:itv}.
\end{proof}

This is a routing-specific, dimension-free specialization of classical
value-of-information / Blackwell comparison, used here as a pre-deployment
screen.

Theorem~\ref{thm:theory-itv} upgrades the design objective into a \emph{router-free, pre-deployment
screen}: a \emph{necessary} condition for a certifiable gain at sample size $m$
is $I_{\mathrm{TV}}(T)\ge 2B(m,\delta)$. It also explains the empirics
mechanistically: an uninformative gate ($I_{\mathrm{TV}}\approx0$) caps
\emph{every} router at $\Phi\approx0$ regardless of oracle headroom
(\S\ref{sec:bank135}).

\subsection{Scope: complementarity is necessary but not sufficient}
\label{sec:theory-T2}

Theorem~\ref{thm:theory-T1} does not characterize \emph{which structural features of a deployment} make
$\Phi>0$---a question a practitioner would want answered before running LLM
experiments. We establish one half, refute a tempting conjecture for the
other, and prove that no Boolean combination of the natural structural
conditions can close the gap. Let (C1) be model complementarity:
writing $\mathcal{A}_j$ for the set of contexts on which advisor $j$ can be
correct, (C1) holds when $\bigcup_j \mathcal{A}_j \supsetneq \mathcal{A}_1$. Let
(C2) be a conditional-informativity requirement on $T$ over the primary-failure
subset.

\begin{theorem}[Scope: necessity of complementarity; refutation of (C2)]
\label{thm:theory-T2}
\textnormal{(N)}\ \emph{(C1) is necessary.} If $\mathcal{A}_j \subseteq
\mathcal{A}_1$ for all $j$, then $\eta_j(t) \le \eta_1(t)$ pointwise, so
$R^\star \equiv 1$ and $G(R^\star) = 0$.\\[2pt]
\textnormal{(R)}\ \emph{(C2) is not necessary.} There exist
latent-type instances satisfying (C1) but violating (C2) on which
$G(R^\star) > 0$. Two such witnesses realize $G(R^\star) = 0.03$ and
$G(R^\star) = 0.10$ respectively, so the conjectured criterion
(C1)$\wedge$(C2) is refuted.
\end{theorem}

\begin{remark}[No Boolean structural NSC exists]
\label{rem:theory-T2}
The witnesses show neither failure-\emph{event} nor failure-\emph{type}
informativeness is necessary, and the gap is provably not repairable in
this vocabulary: two further instances share the full profile
(C1)$\wedge$(C2-type)$\wedge$(C2-flag) yet realize $\Phi=0$ and
$\Phi=0.06>0$, so no Boolean combination of these conditions decides
$\Phi>0$ (Appendix~\ref{app:T2}). What matters is \emph{where} the
informativity points. The operative quantity is the gate's informativeness
$I_{\mathrm{TV}}$---the necessary screen $\Phi\le\tfrac12I_{\mathrm{TV}}$
(Theorem~\ref{thm:theory-itv})---which we adopt as the deployable test.
\end{remark}

Thus complementarity is necessary but not sufficient. An informative gate is
additionally required, with $I_{\mathrm{TV}}$ the structural test
(Rem.~\ref{rem:theory-T2}). On Bank (\S\ref{sec:empirical}), (C1) holds yet
$G{=}{+}0.0$\,pp---the $+9.6$\,pp oracle headroom is \emph{redundant} rather than
error-diverse (\S\ref{sec:theory-indep}) and the gate is uninformative
($\AUC{=}0.525$).

\subsection{The independence baseline: a screen for error diversity}
\label{sec:theory-indep}
Condition (C1) above is \emph{set-theoretic}: it asks only that some non-primary advisor
be correct where the primary fails, which holds whenever the advisors differ at all. The
operationally relevant question is sharper: are the advisors wrong on \emph{different}
inputs \emph{beyond chance}? We make this checkable with a one-line, router-free
diagnostic. For advisors with marginal accuracies $p_1,\dots,p_N$, the oracle accuracy
that \emph{independent} advisors of the same marginals would attain is
$1-\prod_j(1-p_j)$. Define the \emph{excess over independence}
\[
  \mathcal{E} \ :=\ \E[\max_j C_j] \;-\; \big(1-\textstyle\prod_j(1-p_j)\big) .
\]
$\mathcal{E}>0$ signals error diversity beyond what the marginals imply;
$\mathcal{E}\le 0$ means the advisors are statistically \emph{redundant},
co-failing at least as often as chance. Its quantitative role is \emph{additive}
in the routing ceiling:

\begin{lemma}[Ceiling identity]\label{lem:ceiling}
For any gating signal, $\Phi \le A_\star - p_1 = H_{\mathrm{ind}} + \mathcal{E}$,
where $H_{\mathrm{ind}} := 1-\prod_j(1-p_j) - p_1$ is the headroom that
independent advisors of the same marginals would supply.
\end{lemma}
\begin{proof}
$\Phi=\E[\max_j\eta_j(T)]-p_1\le\E\big[\E[\max_j C_j\mid T]\big]-p_1=A_\star-p_1$;
the identity is the definition of $\mathcal{E}$.
\end{proof}

So $\mathcal{E}\le0$ \emph{lowers} the ceiling by $|\mathcal{E}|$ relative to
independent advisors but cannot zero it while $A_\star-p_1>0$. It is
therefore a diagnostic of diversity beyond the marginals rather than a
necessary precondition for positive gain: on Bank,
$H_{\mathrm{ind}}=24.2$\,pp and $\mathcal{E}=-14.6$\,pp, yet a fully
informative gate could still capture $+9.6$\,pp. The realized $G{=}0$ traces
to the \emph{uninformative gate}, while $\mathcal{E}<0$ separately explains
the modest headroom.

\begin{proposition}[Shared difficulty forces redundancy]
\label{prop:redundancy}
Suppose the correctness indicators $C_1,\dots,C_N$ are conditionally independent
given a latent difficulty $D$, and each competence
$q_j(d):=\Prob(C_j{=}1\mid D{=}d)$ is monotone in $d$, all in the \emph{same}
direction. Then the errors are positively associated and $\mathcal{E}\le0$; for
$N{=}2$, $\mathcal{E}=-\mathrm{Cov}(C_1,C_2)\le0$ exactly.
\end{proposition}
\begin{proof}[Sketch]
With $g_j(d){=}1{-}q_j(d)$ (same-direction monotone) and conditioning on $D$,
$\mathcal{E}=\prod_j\E[g_j(D)]-\E[\prod_j g_j(D)]\le0$ by Chebyshev's association
inequality and induction; full proof in Appendix~\ref{app:redundancy}.
\end{proof}
\begin{remark}[Why redundancy is generic]
A shared difficulty axis (the natural consequence of overlapping pretraining
and a common task distribution) thus makes $\mathcal{E}<0$ generic; routable
diversity ($\mathcal{E}>0$) requires competences that move in \emph{opposite}
directions on that axis (genuine specialization), which the $221$-pool screen
never finds.
\end{remark}

\begin{table}[!t]
\caption{\textbf{The redundancy pattern.} $\mathcal{E}\le0$ for every pool tested:
all $221$ RouterBench pools ($55$ pairs, $165$ triples, full $11$; $0$
complementary, best-pool permutation $p{=}1.0$) and all three OpenRCA
distributions. CIs: cell-block bootstrap (RouterBench) / i.i.d.\ bootstrap
(OpenRCA); only the three near-zero pairs straddle $\mathcal{E}{=}0$ (the
full-$11$ CI is wide: anti-correlated cell means inflate the bootstrap
variance).}
\label{tab:redundancy}\centering\small
\renewcommand{\arraystretch}{1.05}
\begin{tabular}{lcr@{\;}l}
\toprule
Advisor pool & $N$ & \multicolumn{2}{c}{$\mathcal{E}$ (pp), $95\%$ CI} \\
\midrule
RouterBench best pair      & 2 & $-0.3$  & $[-3.2,+3.1]$ \\
RouterBench worst pair     & 2 & $-16.4$ & $[-19.4,-11.2]$ \\
RouterBench full pool      & 11 & $-23.4$ & $[-49.0,-3.8]$ \\
Bank (OpenRCA)             & 3 & $-14.6$ & $[-19.2,-9.6]$ \\
Telecom (OpenRCA)          & 6 & $-20.7$ & $[-31.0,-8.2]$ \\
Market cb-1 (OpenRCA)      & 2 & $-0.3$  & $[-4.1,+2.7]$ \\
Market cb-2 (OpenRCA)      & 2 & $-2.4$  & $[-6.8,+1.9]$ \\
\bottomrule
\end{tabular}
\end{table}

Across all $221$ RouterBench pools and the three OpenRCA distributions,
$\mathcal{E}\le0$ (Table~\ref{tab:redundancy}), and even the least-redundant
pool is not significantly diverse (permutation $p{=}1.0$). The positive
control (\S\ref{sec:poscontrol}) operates in the contrasting
$\mathcal{E}>0$ regime.

\subsection{A finite-sample certification bracket}
\label{sec:theory-T3}

We now turn the framework into a deployment guardrail. Given an i.i.d.\ sample of
$m$ instances and a fixed (already-chosen) router $\widehat R$, we ask: can we
certify, at confidence $1-\delta$, that the population gain $G_\mu(\widehat R)$
is positive---robust to a bounded distribution shift of radius $\rho$ between
the evaluation corpus and deployment?

\paragraph{Variance.}
Let $Z_i := C_{\widehat R(T_i),i} - C_{1,i} \in \{-1,0,1\}$ (zero off the route-away
event). With $\sigma_E^2 := \Var(C_{\widehat R(T)} - C_1\mid E_R)$, the per-sample variance is
\[
  \sigma^2 \ :=\ \pi\,\sigma_E^2 + \pi(1-\pi)\,\Delta_E^2 ,
  \qquad \Var(\widehat G_m) = \sigma^2 / m ,
\]
since $\E[Z] = \pi\Delta_E = G_\mu(\widehat R)$. On the $m{=}50$ Bank anchor
(\S\ref{sec:frozen50}), $\sigma^2_{\mathrm{routing}} = 0.0384$, distinct from the null-construction
variance $\sigma^2_{\mu_0} = \pi = 0.12$ ($3\times$ larger)---the variance
that governs the class-level sharp constant.

\paragraph{Bernstein upper bound.}
By a one-sided Bernstein inequality~\citep{bennett1962,blm2013} on the bounded summands $Z_i$, for any
$\delta \in (0,1)$,
\[
  \Prob\!\big( \widehat G_m \ge G_\mu(\widehat R) + B(m,\delta) \big) \le \delta,
\]
\[
  B(m,\delta) \ =\ \sigma\sqrt{\tfrac{2\log(1/\delta)}{m}}
    + \tfrac{2 M \log(1/\delta)}{3 m} ,
\]
which yields the certified lower confidence value
$\widehat\gamma_m := \widehat G_m - B(m,\delta) - 2\rho$ for the
$\rho$-robust gain (the same $B$ bounds the symmetric lower tail, used for
$m^\star$ and the conservatism metric).

\paragraph{Fixed constant; bracket robustness.}
We fix $M{=}2$, the worst-case bound $|Z-\E Z|\le1{+}|G|\le2$ for
$Z\in\{-1,0,1\}$: the constant depends on \emph{no} estimated quantity (the
data-dependent $1{+}|G|$ needs a sample split; the plug-in
$\widehat\sigma^2$ remains, bounded by the empirical-Bernstein bracket
below). The
closed-form $\sqrt{a{+}b}$ relaxation, rather than the choice of $M$, drives
the conservatism: at the anchor the relaxed bracket needs $m^\star{=}312$ while a
direct Bennett inversion (same $M{=}2$) needs $m^\star{=}225$; an
empirical-Bernstein bound~\citep{maurerpontil2009} that
assumes \emph{no} $\sigma^2$ needs $810$. All agree on every reported verdict
(Appendix~\ref{app:bracket}).

\begin{theorem}[Leading-order sharpness of the bracket constant]
\label{thm:theory-T3}
The leading-order constant $c_4=1$ on the $\sqrt{\log(1/\delta)/m}$ term of
$B(m,\delta)$ is \emph{sharp at leading order}: in the
large-deviation/iterated-limit (CLT) regime no smaller constant yields a valid
one-sided bound. At finite $m$ the full bracket remains strictly conservative---the realized
tail at $m^\star$ is ${\approx}1.7{\times}10^{4}$ below the nominal $\delta$
on the audited instance (${\approx}5{\times}10^{2}$ under the direct Bennett
inversion above)---and we make \emph{no} claim of finite-$m$ tightness.
\end{theorem}

\paragraph{Matching Le~Cam lower bound; sharp class constant.}
A Le~Cam two-point construction~\citep{lecam1973,bretagnollehuber1979}
supported on $E_R$ shows no test can certify $G>0$ with substantially fewer
samples: the minimax sample size obeys
$m^\star(\delta,\rho)\asymp\sigma^2\log(1/\delta)/(\pi\Delta_E-2\rho)_+^2$,
i.e.\ $\Theta(\sigma^2\log(1/\delta)/G^2)$ at $\rho{=}0$, matching the
Bernstein upper bound in \emph{scaling}. Over the fixed-activity class
(route-away mass $\pi$ known) the match is \emph{constant-sharp}: the
minimax certification complexity is
$\log(1/\delta)/\KL(\mu_1\|\mu_0)\,(1{+}o(1))$ at the two-point pair,
giving $c^\star_{\mathrm{cert}}{=}2V^\star{=}2\pi$ in the small-gap limit
($8V^\star$ two-sided; $V^\star{=}\pi$ is the extremal mean-zero
variance); the variance-capped certificate attains it, so the protocol is
asymptotically minimax-optimal (Appendix~\ref{app:sharpconst}). The class constant and the
audited instance are governed by different variances, so their sample sizes
differ structurally rather than by slack. The general-class bridge
and the binding $\rho$-regime remain open.

\paragraph{Robustness phase transition.}
The $(\pi\Delta_E - 2\rho)_+^2$ denominator makes $m^\star(\delta,\rho)$ diverge as
$\rho\uparrow\rho^\star=\tfrac12\pi\Delta_E=\tfrac12 G$ ($=0.020$ on the anchor; the
$\pi$ factor is essential, else the radius is overstated by $1/\pi$); below
$\rho^\star$ it is finite.

\paragraph{Learned routers.}\label{sec:theory-afeasible}
When the router is \emph{learned} from $m$ samples, a greedy plug-in selector
has \emph{negative} expected gain on an uninformative gate (optimizer's
curse); a high-confidence threshold router abstains to the
primary, which is the safe action (Appendix~\ref{app:learned}).

\section{Methods}
\label{sec:methods}

This section specifies the evaluation setup and gating signal, the estimation
of the gain functional from a frozen prediction matrix, and the certification
protocol.

\subsection{Benchmark and base-predictor advisors}
\label{sec:setup-bench}

We evaluate on OpenRCA~\citep{openrca2025}, \emph{Bank} dataset. Base predictors
(BPs) are single-agent RCA pipelines, one per Gemini model. After fixing a
pre-compute bug that had silently emptied the telemetry window, we
re-baselined $N{=}3$ BPs on the $m{=}135$ incidents
parseable for all three ($38.5\%$/$28.9\%$/$14.8\%$ strict). The primary ($p_1{=}38.5\%$) and the $48.1\%$ oracle
expose a $+9.6$\,pp headroom, \emph{redundant} rather than error-diverse
($\mathcal{E}={-}14.6$\,pp), whose \emph{capturability} the protocol adjudicates.

\subsection{Deterministic $K{=}4$ KPI-partition gating signal and the partition-support router}
\label{sec:setup-router}

The gating signal is a deterministic $K{=}4$ KPI-family partition computed from
raw telemetry with \emph{no LLM call}: four scouts (CPU, Memory/JVM, Network,
IO/Application) each filter the anomaly table to their family, rank components by
$\sum|z|$, and emit a top-$3$. The outputs are pure functions of the telemetry,
structurally independent of the BP advisors and bit-for-bit reproducible.

The router is the \emph{Partition-Support Router} (PSR), a deterministic selection
$R:\mathcal{T}^N\!\to\![N]$: for each BP it computes a \emph{partition support}
$T_j\in\{0,\dots,K\}$ (the number of scouts whose top-$3$ contains that BP's
root-cause component) and selects the BP of maximal support, breaking ties toward
the primary and defaulting to it when nothing scores. PSR is a pure \emph{argmax}
with \emph{no} inverse-propensity or importance weighting, so the increment stays
$Z=C_{R(T)}-C_1\in\{-1,0,1\}$ with fixed bound $M{=}2$ (\S\ref{sec:theory-T3}). The
default-to-primary makes $G(R)$ depend on the joint law only through $E_R$
(Prop.~\ref{prop:theory-decomp}) and the decision fully auditable.

\subsection{Estimating $G$, $\pi$, $\Delta_E$, and $\sigma^2$ from a frozen matrix}
\label{sec:estimation}

All estimands come from a single \emph{frozen prediction matrix} (hashed,
\S\ref{sec:repro}): with $Z_i:=C_{\widehat R(T_i),i}-C_{1,i}$ we estimate
$\widehat G_m=\tfrac1m\sum_i Z_i$,
$\widehat\pi=\tfrac1m\#\{i:\widehat R(T_i)\neq1\}$, $\widehat\Delta_E$ (mean
increment on routed rows), and
$\widehat\sigma^2=\widehat\pi\widehat\sigma_E^2+\widehat\pi(1-\widehat\pi)\widehat\Delta_E^2$;
these satisfy $\widehat G_m=\widehat\pi\widehat\Delta_E$ exactly, an internal
consistency check.

\subsection{The certification protocol}
\label{sec:protocol}

The protocol replaces ``$\widehat G_m>0$?'' (possibly sampling noise) with a one-sided high-probability lower bound on the population gain.

\paragraph{Certification statistic.} Fix $\delta$ (we use $0.05$) and define the Bernstein bracket
\[
B(m,\delta)\;=\;\sigma\sqrt{\frac{2\log(1/\delta)}{m}}\;+\;\frac{2M\log(1/\delta)}{3m},
\]
where $\sigma^2$ is the routing variance of \S\ref{sec:estimation} and $M{=}2$ is the fixed worst-case bound on the centred increment from \S\ref{sec:theory} (``Fixed constant''); there is no propensity reweighting (\S\ref{sec:setup-router}). The certified gain is $\widehat\gamma_m := \widehat G_m - B(m,\delta)$, and the protocol \textbf{certifies iff $\widehat\gamma_m>0$}; otherwise it \textbf{refuses}. Certification implies $G_\mu(\widehat R)\ge\widehat\gamma_m\ge0$ with probability $\ge1-\delta$. (Modeling a shift radius $\rho$ subtracts a further $2\rho$, collapsing at $\rho^\star=\pi\Delta_E/2$.)

\paragraph{Protocol (pseudocode).} Given increments $\{Z_i\}_{i=1}^m$, level $\delta$, fixed $M{=}2$, shift $\rho$:
\begin{enumerate}\setlength\itemsep{0pt}
  \item $\widehat G_m\leftarrow\tfrac1m\sum_i Z_i$;\ \ $\widehat\sigma^2\leftarrow\widehat\pi\widehat\sigma_E^2+\widehat\pi(1{-}\widehat\pi)\widehat\Delta_E^2$.
  \item $B\leftarrow\widehat\sigma\sqrt{2\log(1/\delta)/m}+2M\log(1/\delta)/(3m)$;\ \ $\widehat\gamma_m\leftarrow\widehat G_m-B-2\rho$.
  \item \textbf{certify} ($G_\mu(\widehat R)\ge\widehat\gamma_m\ge0$ w.p.\ $\ge1{-}\delta$) if $\widehat\gamma_m>0$, \textbf{else refuse}.
  \item \emph{If outcomes cluster, replace $m$ by $m_{\mathrm{eff}}{=}m/\mathrm{deff}$ (condition (iv) below).}
\end{enumerate}

\paragraph{Required sample size.} Inverting $\widehat\gamma_m>0$ gives
$m^\star(\delta)$; on the $m{=}50$ anchor the fixed-$M{=}2$ bracket requires
$m^\star{=}312$ (direct Bennett $225$; all \S\ref{sec:theory-T3} caveats apply).

\paragraph{Validity conditions.} (i)~the empirically best primary is a
data-dependent but \emph{conservative} choice for $G$ (a sample split removes it);
(ii)~the bracket is valid for one pre-specified gate---comparing $k$ gates needs a
held-out split or level $\delta/k$; (iii)~the shift ball
$\{\nu:\TV(\nu,\mu)\le\rho\}$ gives $|G_\nu-G_\mu|\le2\rho$, a conservative
envelope; (iv)~$\sigma^2/m$ is $\Var(\widehat G_m)$ \emph{only} when the
exchangeability unit is the sampled query---when increments cluster (shared
workload cells) the unit is the cluster and we inflate by the design effect
$\mathrm{deff}{=}1+(\bar n_c{-}1)\mathrm{ICC}$ (intraclass correlation), without
which type-I control is lost (\S\ref{sec:routerbench}).

\subsection{Reproducibility}
\label{sec:repro}

All quantitative claims regenerate from frozen, content-hashed artifacts
(artifact index at the end of the appendix): hashed BP prediction tables, a single
source-of-truth constants module (every $G$, $\pi$, $\Delta_E$, variance, $M$,
bracket, $m^\star$), and end-to-end scripts. The deterministic router invokes
no language model. Only the frozen BP predictions ever depended on model
sampling.

\section{Empirical case study: the framework as a guardrail}
\label{sec:empirical}

Three studies exercise the protocol end to end: RouterBench
(\S\ref{sec:routerbench}), OpenRCA
(\S\ref{sec:bank135}--\ref{sec:market}), and a pre-registered positive
control (\S\ref{sec:poscontrol}).

\subsection{Bank-135 re-baseline: redundant advisors, an uninformative gate}
\label{sec:bank135}

On the re-baselined pool (\S\ref{sec:setup-bench}) the router returns
$52/135=38.5\%$, identical to the primary, against an oracle of
$65/135=48.1\%$.

\paragraph{Redundant headroom, uninformative gate.}
The $+9.6$\,pp headroom ($13$ recoverable incidents) means set-theoretic
(C1) holds. The headroom is not error diversity:
it lies $14.6$\,pp below the independence baseline
(Lemma~\ref{lem:ceiling}). An \emph{oracle} gate ($T{=}$ incident identity)
realizes exactly the remaining ceiling,
$G(R^\star){=}A_\star{-}p_1{=}{+}9.6$\,pp
(Thm.~\ref{thm:theory-T1}(a))---capturable in principle on real data; the
two feasible gates capture none (KPI-partition ${+}0.0$, cross-fitted
logistic ${-}3.0$\,pp). Missing headroom therefore does not explain the
oracle-to-feasible gap; gate informativeness does. The gate, however, is uninformative: the
router routes away on mass $\pi{=}0.141$ with $\Delta_E{=}0$ ($2$ recoveries
vs $2$ destructions), so $\widehat G=\pi\Delta_E=+0.0$\,pp. The support
score has $\AUC{=}0.525$ (chance), placing this squarely in the
$\Phi\approx0$ regime of Theorem~\ref{thm:theory-T1}, despite a favorable
marginal $\AUC(T_1;C_1){=}0.605$.

\paragraph{$\widehat\Phi$, a learned gate, and the $I_{\mathrm{TV}}$ screen.}
A cross-fitted plug-in confirms the diagnosis: $\widehat\Phi_{\mathrm{CF}}=-0.015$
($95\%$ CI $[-0.037,0.000]$; a negative point estimate is finite-sample noise on a
true $\Phi{=}0$, the in-sample $+0.022$ being an upward-biased envelope), and a
cross-fitted ridge-logistic gate realizes $\widehat G=-0.030$ $[-0.074,+0.007]$,
the optimizer's curse on an uninformative signal. The raw $\widehat I_{\mathrm{TV}}=0.45$
nominally clears the screen $2B(135,\delta)=0.132$ (fixed $M{=}2$), but a
permutation null reaches $0.34$ from
sparsity alone; the noise-adjusted excess ($+0.106$) sits below $2B$. At this
sample size $I_{\mathrm{TV}}$ therefore functions as a \emph{necessary}
screen whose power grows with $m$, and it is no high-confidence detector.

\paragraph{The protocol correctly refuses.}\label{sec:frozen50}
At $m{=}135$, $\delta{=}0.05$ ($\widehat\sigma{=}0.172$, fixed $M{=}2$) the bracket
is $B{=}0.066$, so $\widehat\gamma_m=-0.066<0$: \textbf{refuse}. With $\widehat G{=}0$
there is no gain to certify at any $m$. The positive-gain anchor would need
$m^\star{=}312>136$, the full Bank count ($135$ parseable), so the guardrail
reaches the correct refusal in both cases. An earlier frozen $m{=}50$ sample
with an apparent $+4$\,pp gain was likewise refused
($\widehat\gamma_{50}=-0.108$): small-sample conservatism rather than
false-positive detection, since at $m{=}50\ll m^\star$ the bracket refuses
\emph{any} $+4$\,pp gain (details in Appendix~\ref{app:conservatism}).

\subsection{Cross-distribution evidence: the same structure on Market}
\label{sec:market}
Re-running on Market (cloudbed-1 (cb-1), $m{=}64$; cb-2, $m{=}78$; identical
pool and gate, Table~\ref{tab:redundancy}; cb-2 a \emph{post-hoc} third
distribution) reproduces the pattern.
\textbf{(i)} The ``primary'' is not universal---Gemini-3.1-pro dominates Bank but
is weaker on cb-1 and ties on cb-2 (McNemar $p{=}0.84$), itself an
argument for selection. \textbf{(ii)} The oracle headrooms
($+9.6$/$+10.9$/$+14.1$\,pp for Bank/cb-1/cb-2) all sit at or below their
independence baselines (excess $-14.6$/$-0.3$/$-2.4$\,pp): redundancy, not
complementarity (Lemma~\ref{lem:ceiling}). \textbf{(iii)} The KPI-partition gate
does not transfer to Market's microservice metrics, so the realized gain stays
$\approx0$, as Theorem~\ref{thm:theory-itv} predicts. A positive certification
needs error diversity \emph{and} an informative gate. Redundancy lowers the
ceiling but does not preclude it; the next subsection constructs the contrasting
regime.

\begin{figure*}[t]\centering
\includegraphics[width=0.36\textwidth]{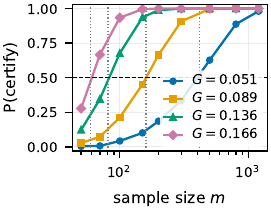}\hspace{0.05\textwidth}%
\includegraphics[width=0.36\textwidth]{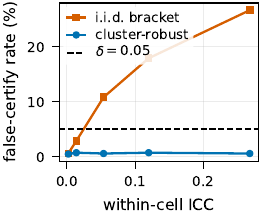}\\[-1pt]
{\footnotesize\makebox[0.36\textwidth]{(a) power}\hspace{0.05\textwidth}%
\makebox[0.36\textwidth]{(b) type-I under clustering}}
\caption{\textbf{Calibration} (\S\ref{sec:poscontrol}). \textbf{(a)}~certify-rate
vs $m$ for several true gains; dotted lines mark $m^\star(G)$.
\textbf{(b)}~false-certify rate on a true null vs ICC: the i.i.d.\ bracket
loses type-I control; the cluster-robust bracket holds $\le\delta$.}
\label{fig:calib}
\end{figure*}

\subsection{A real-data certification, and what survives clustering}
\label{sec:routerbench}
We apply the protocol, pre-registered and unchanged, to
RouterBench~\citep{routerbench2024}: $36{,}497$ prompts scored for $11$
cross-family models, strict $C_j=\one\{\text{score}{=}1\}$, with \emph{workload
identity} ($86$ task cells) as the deployable gate. On a seeded stratified
50/50 split, the train-fitted Bayes router routes away from the GPT-4 primary on
$3$ cells; on the frozen test half ($m{=}18{,}230$) the realized gain is
$\widehat G={+}0.0050$ and the prompt-i.i.d.\ bracket certifies
($\widehat\gamma^{\mathrm{iid}}={+}0.0034>0$, $m^\star{=}2{,}855$): a real
${+}0.5$\,pp gain over a $64.2\%$ primary---\emph{if} fresh queries are
exchangeable at the prompt level (the same workload mix).

\paragraph{Dependence on the exchangeability unit.} That gain lives entirely
in $3$ of $86$ workload cells, and outcomes correlate within a cell
($\mathrm{ICC}{=}0.15$). If deployment may draw \emph{new} workload types the
resampling unit is the cell rather than the prompt. Inflating the bracket by
the design effect ($\mathrm{deff}{=}29.8$, $m_{\mathrm{eff}}{=}612$) flips
the verdict to \emph{withhold} ($\widehat\gamma^{\mathrm{clu}}={-}0.0094$),
and a cell-block bootstrap resampling whole cells gives a $95\%$ CI on the
gain of $[{-}0.0014,{+}0.0142]$ that straddles zero, a bracket-independent
diagnostic locating the fragility in the gain itself. Certification holds
in-workload and is withheld for cross-workload generalization; beating the
primary requires the error diversity that Table~\ref{tab:redundancy} shows
is absent.

\paragraph{Which certifications survive clustering.} Of all $11$ candidate
primaries, $10$ certify under the cluster bracket, but each draws $96$--$98\%$
of its gain from routing to GPT-4, a trivial base-model upgrade under our accuracy-only
scope. GPT-4 is the \emph{only} primary whose gain spreads across
multiple advisors (top advisor $58\%$), and the one whose gain is withheld:
the protocol certifies the obvious and withholds the non-robust
complementarity. Both verdicts replicate 5-shot
($\widehat\gamma^{\mathrm{iid}}={+}0.0025$, $\widehat\gamma^{\mathrm{clu}}={-}0.0086$,
gain in $7/86$ cells) and the graded-score gain (${+}0.17$\,pp) is correctly
refused (no new LLM calls).

\subsection{A calibrated positive control}
\label{sec:poscontrol}
A pre-registered semi-synthetic control supplies a \emph{known} ground-truth
gain and a true null, which no real benchmark provides: a latent regime
drives two conditionally independent advisors and an $85\%$-accurate gate, an
error-diverse instance ($\mathcal{E}={+}3.8$\,pp, error correlation $\phi={-}0.17$,
$\AUC=0.75$) with $\widehat G=0.103$, $m^\star=146$. The protocol behaves as a calibrated guardrail
(seed $20260608$, $2000$ MC draws, exact $95\%$ CIs): \textbf{(a)}~certify at
$m{=}300\ge m^\star$ (certify-rate $98.95\%$ $[98.40,99.35]$);
\textbf{(b)}~withhold the same gain at $m{=}70<m^\star$ (certify-rate
$11.10\%$ $[9.76,12.56]$);
\textbf{(c)}~type-I---a matched null ($G{=}0$, same gate) certifies $0.15\%$
$[0.03,0.44]\le\delta$. It refuses on OpenRCA because those pools are
redundant and the sample small.

\paragraph{Power and the cluster-robust upgrade.}
Figure~\ref{fig:calib} characterizes the protocol beyond these point
checks. \emph{Power:} sweeping $G$ and $m$, the certify-rate
is $\approx0$ below $m^\star(G)$, $\tfrac12$ at $m^\star$, and $\to1$ above:
$m^\star$ is the operating threshold. \emph{Type-I under
correlation:} on a true null with tunable ICC the i.i.d.\ bracket's false-certify
rate climbs from $0.6\%$ at $\mathrm{ICC}{=}0$ to $27\%$ at $\mathrm{ICC}{=}0.27$
($\gg\delta$), while the design-effect cluster bracket of
\S\ref{sec:routerbench} holds $\le\delta$ throughout (and does not
over-correct at $\mathrm{ICC}{=}0$).

\subsection{Limitations}
\label{sec:limitations}
Four limits bound our claims. (i)~\emph{Redundancy sources are only partly
modeled}: Prop.~\ref{prop:redundancy} explains $\mathcal{E}\le0$ via shared
difficulty, which we do not stratify empirically.
(ii)~\emph{Gates}: OpenRCA uninformativeness is shown for the PSR and a
cross-fitted logistic gate, and RouterBench's only beat-the-best gain
(${+}0.5$\,pp, $3/86$ cells) is withheld under clustering---richer gates may
surface a robustly certifiable slice. (iii)~\emph{Shift}: the TV ball is
conservative and at small gains $\rho^\star{=}G/2$ tolerates little.
(iv)~\emph{Accuracy-only scope}: cost-aware extensions compose with the same
bracket.

\section{Conclusion}
\label{sec:conclusion}
\textbf{RouteGuard} decides before shipping whether routing among LLM
advisors yields a warranted gain. The design objective is the gate's
informativeness $\Phi$ rather than its AUC (Theorem~\ref{thm:theory-T1}),
and a finite-sample bracket, constant-sharp at class level
($c^\star_{\mathrm{cert}}{=}2V^\star$), certifies the gain or refuses.
Every pool we tested---$221$ RouterBench pools and three OpenRCA
distributions---is redundant ($\mathcal{E}\le0$), as shared difficulty
predicts (Prop.~\ref{prop:redundancy}). RouterBench certifies a GPT-4 gain
under prompt sampling and withholds it under workload clustering. OpenRCA
refuses on an uninformative gate, and a pre-registered control confirms
calibration. Open problems include richer gates for the redundancy-lowered
ceiling, the general-class constant bridge, and the binding $\rho$-regime.

\appendices

\section*{Overview of the Technical Appendices}
\noindent This appendix contains the complete proofs, constructions, and
numerical specifications for every result stated in the main paper. Main-paper
results are referred to by descriptive name (no cross-document numbering).
Coverage map:
\begin{itemize}\setlength\itemsep{1pt}
  \item \S\ref{app:notation}: notation and standing assumptions.
  \item \S\ref{app:T1}: full proof of the \emph{design-objective theorem} (T1),
        including both AUC-insufficiency witnesses.
  \item \S\ref{app:itv}: full proof of the \emph{informativeness-ceiling
        theorem} $\Phi\le\tfrac12 I_{\mathrm{TV}}(T)$, the exact equality
        characterization with conditions (i)--(iii), and the
        mutual-information companion.
  \item \S\ref{app:T2}: the explicit latent-type constructions behind the
        \emph{T2 theorem}---the two witnesses ($G(R^\star)=0.03$, $0.10$)
        and the null--positive pair proving no Boolean structural NSC
        exists---with full joint laws and computations.
  \item \S\ref{app:lecam}: the Le~Cam two-point lower bound with the
        Bretagnolle--Huber refinement, the general sharp-constant
        theorems (least-favourable-pair form), and their specialization
        to the fixed-activity class ($c^\star_{\mathrm{cert}}{=}2V^\star$,
        two-sided $8V^\star$; \S\ref{app:sharpconst}).
  \item \S\ref{app:c4}: full proof of the \emph{leading-order sharpness
        theorem} ($c_4=1$) for the certification bracket.
  \item \S\ref{app:conservatism}: precise definition and computation of the
        finite-$m$ conservatism figure (${\approx}1.7{\times}10^4$ at
        $m^\star$) quoted in the main paper.
  \item \S\ref{app:learned}: the learned-router guardrail (details promised
        in the main paper's ``Learned routers'' paragraph).
  \item \S\ref{app:poscontrol}: full generator specification, pre-registration
        record, and exact Clopper--Pearson confidence intervals for the
        calibrated positive control.
  \item \S\ref{app:redundancy}: full proof of the \emph{shared-difficulty
        redundancy proposition} promised in the main paper.
  \item \S\ref{app:bracket}: the bracket-robustness comparison (fixed $M$,
        direct Bennett, empirical Bernstein) behind the main paper's
        ``all agree on every reported verdict.''
  \item \S\ref{app:veriftable}: the consolidated numerical-verification
        table (every check in one place).
\end{itemize}

\section{Notation and standing assumptions}
\label{app:notation}
Table~\ref{tab:notation} summarizes notation.

We use the main paper's setup. There are $N\ge2$ advisors with correctness
indicators $C_j\in\{0,1\}$ and marginals $p_j:=\E[C_j]$, indexed so that the
\emph{primary} is $j{=}1$ with $p_1=\max_j p_j$. The gating signal $T$ takes
values in a finite alphabet (in the system, $T=(T_1,\dots,T_N)$ with
$T_j\in\{0,\dots,K\}$); $\mu_T$ is its law and
$\eta_j(t):=\Prob(C_j{=}1\mid T{=}t)$ the posterior correctness. A router is a
measurable $R:\mathcal{T}^N\to[N]$ with accuracy $A_R:=\E[C_{R(T)}]$ and gain
$G(R):=A_R-p_1$; the routing-gain decomposition proposition of the main paper
gives $G(R)=\pi\Delta_E$ with $\pi:=\Prob(R(T)\ne1)$ and
$\Delta_E:=\E[C_{R(T)}-C_1\mid R(T)\ne1]$. The informativeness functional and
the total-variation informativeness are
\begin{align*}
  \Phi&:=\E\big[\max_j\eta_j(T)-\eta_1(T)\big],\\
  I_{\mathrm{TV}}(T)&:=\E\|\eta(T)-p\|_1 ,
\end{align*}
where $\eta(t):=(\eta_1(t),\dots,\eta_N(t))$ is the posterior-accuracy vector,
$p:=(p_1,\dots,p_N)$ is the vector of marginal accuracies, and both
(unconditional) expectations average over $T$. Componentwise
$\eta_j(T)=\E[C_j\mid T]$ and $p_j=\E[C_j]$, so equivalently
$I_{\mathrm{TV}}(T)=\E\bigl\|\E[C\mid T]-\E[C]\bigr\|_1$ with
$C:=(C_1,\dots,C_N)$.
Throughout, $W_j(t):=\eta_j(t)-p_j$ denotes the posterior deviation; by the
tower property $\E[W_j(T)]=0$ for every $j$.

Three information measures recur. For probability measures $P,Q$ on a common
finite (or countable) set with $P\ll Q$, the Kullback--Leibler divergence is
$\KL(P\|Q):=\sum_xP(x)\log\tfrac{P(x)}{Q(x)}$ (natural logarithm; all
information quantities are in nats). For discrete random variables $X,Y$, the
mutual information is
$\MI(X;Y):=\KL\bigl(\mathcal L(X,Y)\,\|\,\mathcal L(X)\otimes\mathcal L(Y)\bigr)
=\E\bigl[\KL\bigl(\mathcal L(Y\mid X)\,\|\,\mathcal L(Y)\bigr)\bigr]$, the
expected divergence of the conditional from the marginal law. The total
variation distance is $\TV(P,Q):=\sup_A|P(A)-Q(A)|
=\tfrac12\sum_x|P(x)-Q(x)|$. The certification bracket is
\[
  B(m,\delta)=\sigma\sqrt{\tfrac{2\log(1/\delta)}{m}}
              +\tfrac{2M\log(1/\delta)}{3m},
\]
with $\sigma^2=\pi\sigma_E^2+\pi(1-\pi)\Delta_E^2$ the routing variance of
$Z:=C_{R(T)}-C_1\in\{-1,0,1\}$ (where $\sigma_E^2:=\Var(C_{R(T)}-C_1\mid E_R)$)
and $M=2$ the fixed worst-case bound on the
centred increment $|Z-\E Z|\le1+|G|\le2$. The certified gain is
$\widehat\gamma_m=\widehat G_m-B(m,\delta)-2\rho$, with $\rho$ the assumed
total-variation shift radius. The finite-alphabet
assumption makes every $\argmax$ selection trivially measurable; all proofs
extend to general measurable $T$ by replacing sums with integrals and fixing a
measurable selection of the $\argmax$ (which exists by standard selection
theorems on countably generated spaces).

\begin{table}[!t]
\caption{Notation summary.}
\label{tab:notation}\centering\footnotesize
\renewcommand{\arraystretch}{1.15}
\begin{tabular}{@{}ll@{}}
\toprule
\textbf{Symbol} & \textbf{Definition}\\
\midrule
$C_j$ & advisor-$j$ correctness, $\in\{0,1\}$\\
$p_j$ & $\E[C_j]$, marginal accuracy\\
$p_1$ & $\max_j p_j$, primary accuracy\\
$A_\star$ & $\E[\max_j C_j]$, oracle accuracy\\
$T$ & gating signal, $T_j\in\{0,\dots,K\}$\\
$\eta_j(t)$ & $\Prob(C_j{=}1\mid T{=}t)$, posterior correctness\\
$R$ & router $\mathcal{T}^N\to[N]$ (measurable)\\
$G(R)$ & $A_R-p_1$, routing gain\\
$E_R$ & $\{t:R(t)\ne1\}$, route-away event\\
$\pi$ & $\Prob(E_R)$, route-away mass\\
$\Delta_E$ & $\E[C_{R(T)}-C_1\mid E_R]$, conditional edge\\
$\Phi$ & $\E[\max_j\eta_j(T)-\eta_1(T)]$, informativeness\\
$p$, $\eta(t)$ & vectors $(p_j)_{j\le N}$, $(\eta_j(t))_{j\le N}$\\
$I_{\mathrm{TV}}(T)$ & $\E\|\eta(T)-p\|_1$, TV informativeness\\
$\mathcal{E}$ & $A_\star-(1-\prod_j(1-p_j))$, excess over indep.\\
$B(m,\delta)$ & Bernstein certification bracket\\
$\widehat\gamma_m$ & $\widehat G_m-B-2\rho$, certified gain\\
$\sigma^2$ & routing variance of $Z=C_{R(T)}-C_1$\\
$M$ & bound on $|Z-\E Z|$; worst case $2$\\
$m^\star$ & minimum certifying sample size\\
$\rho^\star$ & $\pi\Delta_E/2$, robustness phase transition\\
\bottomrule
\end{tabular}
\end{table}

\section{Proof of the design-objective theorem (T1)}
\label{app:proofs}
\label{app:T1}

\begin{theorem}[design objective; the \emph{design objective} theorem (T1) of the
main paper]\label{thm:t1}
\textnormal{(a)} For \emph{every} joint law $\mathcal{L}$,
$\max_R G(R)=\Phi(\mathcal{L})$, attained by the Bayes selector
$R^\star(t)=\argmax_j\eta_j(t)$; \textnormal{(b)} on the nuisance class
$\mathfrak{L}_{p_1,A_\star}=\{\mathcal{L}:\E[C_1]=p_1,\ \E[\max_jC_j]=A_\star\}$
there exist
$\mathcal{L}_a,\mathcal{L}_b\in\mathfrak{L}_{1/2,\,3/4}$ with identical
$\AUC(T_1;C_1)=\tfrac12$ but $\Phi(\mathcal{L}_a)=0\ne\tfrac14=\Phi(\mathcal{L}_b)$.
\end{theorem}

\begin{proof}[Proof of (a)]
For any router $R$, conditioning on $T$ and applying the tower property,
\[
  A_R=\E\big[\E[C_{R(T)}\mid T]\big]=\E\big[\eta_{R(T)}(T)\big]
     \le\E[\max_j\eta_j(T)],
\]
where the inequality holds pointwise in $t$. The Bayes selector
$R^\star(t)\in\argmax_j\eta_j(t)$ (measurable on the finite alphabet) attains
the pointwise maximum, so
$\max_R A_R=\E[\max_j\eta_j(T)]$. Subtracting the constant $p_1=\E[\eta_1(T)]$
on both sides gives $\max_R G(R)=\E[\max_j\eta_j(T)-\eta_1(T)]=\Phi$.
\end{proof}

\begin{proof}[Proof of (b): the two witnesses]
Both witnesses have $N{=}2$, a binary scalar gate (formally $K{=}1$ with
$T=(T_1,T_2)$, $T_2\equiv0$ degenerate, so the router sees the Bernoulli
component $T_1$, written $T$ below), and identical nuisance
$p_1=p_2=\tfrac12$, $A_\star=\tfrac34$.

\emph{Witness $\mathcal{L}_a$ (independent gate).} Let $C_1,C_2$ be i.i.d.\
$\mathrm{Bernoulli}(\tfrac12)$ and $T\sim\mathrm{Bernoulli}(\tfrac12)$
independent of $(C_1,C_2)$. Then $p_1=p_2=\tfrac12$ and
$A_\star=\Prob(\max(C_1,C_2)=1)=1-\tfrac14=\tfrac34$, so
$\mathcal{L}_a\in\mathfrak{L}_{1/2,3/4}$. Independence gives
$\eta_1(t)=\eta_2(t)\equiv\tfrac12$, hence $\Phi(\mathcal{L}_a)=0$; and
$T\perp C_1$ gives $\AUC(T;C_1)=\tfrac12$.

\emph{Witness $\mathcal{L}_b$ (maximally informative gate, invisible to AUC).}
Place mass $\tfrac14$ on each of the four atoms of $(T,C_1,C_2)$:
\[
  (0,1,1),\quad(0,0,1),\quad(1,1,0),\quad(1,0,0).
\]
Marginals: $p_1=\Prob(C_1{=}1)=\tfrac14+\tfrac14=\tfrac12$ and
$p_2=\tfrac12$; the oracle is correct on the first three atoms, so
$A_\star=\tfrac34$ and $\mathcal{L}_b\in\mathfrak{L}_{1/2,3/4}$. The gate is
independent of the \emph{primary's} correctness:
$\Prob(T{=}1\mid C_1{=}1)=\Prob(T{=}1\mid C_1{=}0)=\tfrac12$, so
$\AUC(T;C_1)=\tfrac12$, identical to $\mathcal{L}_a$. Yet the posteriors are
$\eta_1(0)=\eta_1(1)=\tfrac12$, $\eta_2(0)=1$, $\eta_2(1)=0$: the gate
\emph{perfectly} reveals when the competitor is correct. Hence
\[
  \Phi(\mathcal{L}_b)
  =\tfrac12\big(\max(\tfrac12,1)-\tfrac12\big)
  +\tfrac12\big(\max(\tfrac12,0)-\tfrac12\big)
  =\tfrac14 .
\]
Identical nuisance, identical $\AUC=\tfrac12$, distinct $\Phi\in\{0,\tfrac14\}$:
$\AUC(T_1;C_1)$ is not a sufficient statistic for $\Phi$.
\end{proof}

\begin{remark}[a one-parameter family; numerical verification]
The mixture $\mathcal{L}_\alpha:=(1-\alpha)\mathcal{L}_a+\alpha\mathcal{L}_b$,
$\alpha\in[0,1]$, keeps $p_1=p_2=\tfrac12$, $A_\star=\tfrac34$, and
$\AUC(T;C_1)=\tfrac12$ fixed for every $\alpha$ while
$\Phi(\mathcal{L}_\alpha)=\alpha/4$ sweeps $[0,\tfrac14]$ (direct computation;
verified numerically at $\alpha\in\{0,\tfrac14,\tfrac12,\tfrac34,1\}$, where
$\Phi=0,0.0625,0.125,0.1875,0.25$ to machine precision). Part~(a) was
additionally verified by brute force on $200$ random joint laws ($N{=}2$,
$K{=}2$): enumerating all $2^9=512$ routers, $\max_RG(R)$ matched $\Phi$ in
$200/200$ runs with maximum deviation $2.9\times10^{-16}$.
\end{remark}

\section{Full proof of the informativeness ceiling}
\label{app:itv}

This section proves the main paper's \emph{informativeness-ceiling theorem}
($\Phi\le\tfrac12 I_{\mathrm{TV}}(T)$), establishes the exact equality
characterization behind its conditions (i)--(iii), and proves the
mutual-information companion.

\begin{lemma}[mean-zero positive-part identity]\label{lem:meanzero}
If $W$ is integrable with $\E[W]=0$, then $\E[W_+]=\tfrac12\E|W|$, where
$W_+:=\max(W,0)$.
\end{lemma}
\begin{proof}
$W_+=\tfrac12(|W|+W)$ pointwise; take expectations and use $\E[W]=0$.
\end{proof}

Applied to $W_j(T)=\eta_j(T)-p_j$ (mean zero by the tower property),
Lemma~\ref{lem:meanzero} gives the form of the ceiling we actually prove
against:
\begin{equation}\label{eq:itvhalf}
  \tfrac12 I_{\mathrm{TV}}(T)
  =\tfrac12\sum_{j=1}^N\E\big|W_j(T)\big|
  =\sum_{j=1}^N\E\big[(W_j(T))_+\big].
\end{equation}

\begin{theorem}[informativeness ceiling; the ``informativeness ceiling on
capturable gain'' theorem of the main paper]\label{thm:itv}
With the primary indexed so that $p_1=\max_jp_j$,
\[
  \Phi\;=\;\max_RG(R)\;\le\;\tfrac12\,I_{\mathrm{TV}}(T).
\]
\end{theorem}

\begin{proof}
By Theorem~\ref{thm:t1}(a) and $\E[\eta_1(T)]=p_1$,
\[
  \Phi=\E\big[\max_j\eta_j(T)\big]-p_1
      =\sum_t\mu_T(t)\Big(\max_j\eta_j(t)-p_1\Big).
\]
Fix $t$ and let $j^\star=j^\star(t)\in\argmax_j\eta_j(t)$ be the \emph{maximizing}
advisor at $t$ (any selection; the bound below does not depend on it). Split
the maximum through the maximizing advisor's deviation and drop the necessarily
non-positive marginal gap:
\begin{align}
  \max_j\eta_j(t)-p_1
  &=W_{j^\star}(t)+\big(p_{j^\star}-p_1\big)\notag\\
  &\le W_{j^\star}(t)
   \tag{drop $p_{j^\star}-p_1\le0$}\\
  &\le \big(W_{j^\star}(t)\big)_+
   \;\le\;\sum_{j=1}^N\big(W_j(t)\big)_+ .
   \label{eq:pointwise}
\end{align}
Integrating \eqref{eq:pointwise} against $\mu_T$ and applying
\eqref{eq:itvhalf},
\[
  \Phi\le\sum_j\E\big[(W_j(T))_+\big]=\tfrac12 I_{\mathrm{TV}}(T).\qedhere
\]
\end{proof}

\subsection{Exact equality characterization and conditions (i)--(iii)}

The proof of Theorem~\ref{thm:itv} bounds, cell by cell, the
selection-independent quantity $\max_j\eta_j(t)-p_1$ by
$\sum_j(W_j(t))_+$. Equality of the integrals therefore holds exactly when the
pointwise bound is tight almost everywhere:

\begin{theorem}[equality characterization]\label{thm:eq}
$\Phi=\tfrac12 I_{\mathrm{TV}}(T)$ if and only if, for $\mu_T$-a.e.\ $t$,
\begin{equation}\label{eq:star}
  \max_j\eta_j(t)\;=\;p_1+\sum_{j=1}^N\big(\eta_j(t)-p_j\big)_+ .
\end{equation}
Condition \eqref{eq:star} holds at $t$ if and only if:
\begin{itemize}\setlength\itemsep{1pt}
  \item[\textnormal{(E1)}] at most one advisor $j$ has a strictly positive
        deviation $\eta_j(t)>p_j$ at $t$; and
  \item[\textnormal{(E2)}] if such an advisor exists it attains the maximum
        $\eta_j(t)=\max_k\eta_k(t)$ and is mean-tied to the primary,
        $p_j=p_1$; if none exists, $\max_k\eta_k(t)=p_1$.
\end{itemize}
\end{theorem}

\begin{proof}
\emph{(Equality $\Leftrightarrow$ \eqref{eq:star} a.e.)} From the proof of
Theorem~\ref{thm:itv},
\[
  \Phi-\tfrac12I_{\mathrm{TV}}
  =\sum_t\mu_T(t)\Big[\max_j\eta_j(t)-p_1-\sum_j(W_j(t))_+\Big],
\]
and each bracket is $\le0$ by \eqref{eq:pointwise}. A sum of non-positive
terms vanishes iff each term vanishes on the support of $\mu_T$, which is
\eqref{eq:star}.

\emph{(\eqref{eq:star} $\Leftrightarrow$ (E1)--(E2)).} Suppose
\eqref{eq:star} holds at $t$ and let $j^\star$ be any advisor attaining the
maximum. Then
\[
  W_{j^\star}(t)+p_{j^\star}
  =p_1+\sum_j(W_j(t))_+
  \ge p_1+(W_{j^\star}(t))_+ .
\]
Since also $p_{j^\star}\le p_1$ and $W_{j^\star}\le(W_{j^\star})_+$, both
inequalities must be equalities: $p_{j^\star}=p_1$,
$W_{j^\star}(t)=(W_{j^\star}(t))_+\ge0$, and
$\sum_{j\ne j^\star}(W_j(t))_+=0$, i.e.\ no advisor other than $j^\star$
deviates positively. If $W_{j^\star}(t)>0$ this is exactly (E1)--(E2) with the
positive deviator attaining the maximum; if $W_{j^\star}(t)=0$ then no advisor
deviates positively and $\max_k\eta_k(t)=p_{j^\star}=p_1$. Conversely, if
(E1)--(E2) hold at $t$ with positive deviator $j_t$ (mean-tied, attaining the
maximum), then $\max_j\eta_j(t)=W_{j_t}(t)+p_1$ and
$\sum_j(W_j)_+=W_{j_t}(t)$, so \eqref{eq:star} holds; if no positive deviator
exists, $\sum_j(W_j)_+=0$ and $\max_k\eta_k(t)=p_1$ gives \eqref{eq:star}
directly.
\end{proof}

\begin{corollary}[the canonical conditions (i)--(iii)]\label{cor:i-iii}
Consider the three conditions quoted in the main-paper theorem statement:
\textnormal{(i)} $\eta_1\equiv p_1$ a.e.; \textnormal{(ii)} at most one
competitor $j\ne1$ is active (deviates positively, $\eta_j(t)>p_j$) at each
$t$; \textnormal{(iii)} every active competitor is mean-tied to the primary,
$p_{j^\star}=p_1$. Then:
\begin{itemize}\setlength\itemsep{1pt}
  \item[\textnormal{(a)}] (i)--(iii) jointly imply
        $\Phi=\tfrac12I_{\mathrm{TV}}(T)$;
  \item[\textnormal{(b)}] under the normalization (i), equality holds iff
        (ii)--(iii) hold; thus (i)--(iii) is the canonical (primary-flat)
        form of the characterization \eqref{eq:star};
  \item[\textnormal{(c)}] if the primary is the \emph{unique} marginal
        maximiser ($p_1>p_j$ for all $j\ne1$), equality is possible only in
        the degenerate case $I_{\mathrm{TV}}(T)=0$ (all posteriors flat), and
        then (i) holds automatically. Non-degenerate equality requires a
        competitor mean-tied to the primary.
\end{itemize}
\end{corollary}

\begin{proof}
(a) Under (i), $W_1\equiv0$, so the positive deviators are exactly the active
competitors; (ii)--(iii) then give (E1)--(E2) at every $t$ (at cells without
an active competitor, $\eta_j(t)\le p_j\le p_1=\eta_1(t)$, so
$\max_k\eta_k(t)=p_1$), and Theorem~\ref{thm:eq} applies. (b) Under (i),
(E1)--(E2) reduce verbatim to (ii)--(iii). (c) By Theorem~\ref{thm:eq},
any cell with a positive \emph{competitor} deviator forces a mean tie
$p_j=p_1$; if no ties exist, no competitor deviates positively at any cell,
so $W_j\le0$ a.e.\ for all $j\ne1$. The primary cannot deviate positively
either: if $W_1(t_0)>0$ somewhere, $\E[W_1(T)]=0$ gives a cell $t_1$ with
$W_1(t_1)<0$, where \eqref{eq:star} forces $\max_k\eta_k(t_1)=p_1$ (all positive parts
vanish there), so some competitor $k$ attains $\eta_k(t_1)=p_1>p_k$---a
positive competitor deviation, just excluded. Hence $W_j\le0$ a.e.\ for
\emph{all} $j$; combined with $\E[W_j(T)]=0$ this forces $W_j\equiv0$ a.e.,
i.e.\ $I_{\mathrm{TV}}(T)=0$ (and $\Phi=0$).
\end{proof}

\begin{remark}[clause (iii) cannot be dropped; (i) is a normalization]
\label{rem:witnesses}
Two explicit two-cell instances ($\mu_T$ uniform on $\{t_1,t_2\}$, $N{=}2$)
pin down the role of each clause; both were verified numerically.
\emph{Equality witness} (satisfies (i)--(iii)): $\eta_1=(0.6,0.6)$,
$\eta_2=(0.9,0.3)$, so $p_1=p_2=0.6$ (mean-tied) and the single active
competitor sits at $t_1$. Then $\Phi=\tfrac12(0.9-0.6)=0.15$ and
$\tfrac12I_{\mathrm{TV}}=\tfrac12\,\E|W_2|=\tfrac12(0.3)=0.15$ (here
$W_1\equiv0$): ratio exactly $1$.
\emph{Mean-tie violated} (a case on which the uncorrected ``iff'' claim fails): $\eta_1=(0.6,0.6)$,
$\eta_2=(0.95,0.15)$, so $p_2=0.55<p_1$. Then $\Phi=\tfrac12(0.95-0.6)=0.175$
while $\tfrac12I_{\mathrm{TV}}=0.2$: ratio $0.875<1$, confirming that without
clause (iii) equality fails. Finally, (i) is a \emph{normalization} rather
than a consequence of bare equality: when the primary is mean-tied with a
competitor, label-symmetric equality cases exist in which the primary itself
carries the positive deviation on cells where it attains the maximum---e.g.\
$\eta_1=(0.7,0.3)$, $\eta_2=(0.4,0.6)$, $p_1=p_2=0.5$ gives
$\Phi=0.15=\tfrac12I_{\mathrm{TV}}$ with $\eta_1\not\equiv p_1$, consistent
with Theorem~\ref{thm:eq} (each cell has exactly one positive deviator,
mean-tied and maximal). The main-paper statement is to be read in the
canonical form of Corollary~\ref{cor:i-iii}(b).
\end{remark}

\subsection{The mutual-information companion}

\begin{proposition}[MI companion]\label{prop:mi}
With $\MI$ in nats,
\[
\begin{aligned}
  \Phi\;\le\;\tfrac12 I_{\mathrm{TV}}(T)
      \;&\le\;\frac{1}{2\sqrt2}\sum_{j=1}^N\sqrt{\MI(T;C_j)}\\
      &\le\;\frac{1}{\sqrt2}\sum_{j=1}^N\sqrt{\MI(T;C_j)} .
\end{aligned}
\]
The middle constant $\tfrac1{2\sqrt2}$ is the one the proof yields; the
right-most $\tfrac1{\sqrt2}$ follows by a trivial final bound.
\end{proposition}

\begin{proof}
For binary $C_j$, the total variation between the conditional and marginal
laws of $C_j$ is
$\TV\big(\mathcal{L}(C_j\mid T{=}t),\mathcal{L}(C_j)\big)=|\eta_j(t)-p_j|$.
Pinsker's inequality $\TV(P,Q)\le\sqrt{\KL(P\|Q)/2}$ gives, pointwise in $t$,
\[
  |\eta_j(t)-p_j|\le
  \sqrt{\tfrac12\KL\big(\mathcal{L}(C_j\mid T{=}t)\,\big\|\,\mathcal{L}(C_j)\big)} .
\]
Averaging over $T$ and applying Jensen's inequality to the concave square
root,
\[
\begin{aligned}
  \E\big|W_j(T)\big|
  &\le\sqrt{\tfrac12\,\E\Bigl[\KL\big(\mathcal{L}(C_j\mid T)\,\|\,\mathcal{L}(C_j)\big)\Bigr]}\\
  &=\sqrt{\tfrac12\MI(T;C_j)},
\end{aligned}
\]
since the expected conditional-vs-marginal KL is exactly the mutual
information (\S\ref{app:notation}). Summing over $j$,
$I_{\mathrm{TV}}(T)\le\sum_j\sqrt{\MI(T;C_j)/2}$, and Theorem~\ref{thm:itv}
multiplies by $\tfrac12$. The final inequality is trivial.
\end{proof}

\begin{remark}[numerical verification and the Bank-135 instantiation]
\label{rem:itv-numerics}
\emph{Verification.} Over $2\times10^5$ random joint laws ($N{=}3$ advisors,
$|\mathcal T|=4$ gate values, Dirichlet-uniform $\mu_T$, uniform $\eta$,
primary set to the marginal-best advisor as the theorem requires):
$0$ violations of $\Phi\le\tfrac12I_{\mathrm{TV}}$, with supremum ratio
$\Phi/(\tfrac12I_{\mathrm{TV}})=0.9956$; the equality witness of
Remark~\ref{rem:witnesses} attains ratio $1.0000$ exactly.
\emph{Instantiation.} On the frozen Bank-135 correctness matrix with the
partition-support gate ($135$ rows, $20$ occupied gate cells), the
\emph{in-sample optimal} router given this gate---an overfit ceiling, since it
re-uses the evaluation rows---captures
$\widehat\Phi_{\mathrm{achievable}}=0.0222$ of the $+0.096$ oracle headroom,
and the inequality holds empirically
($0.0222\le\tfrac12\widehat I_{\mathrm{TV}}=0.2237$; the plug-in
$\widehat I_{\mathrm{TV}}$ is sharply upward-biased here because $20$ cells
share $135$ rows, so its absolute value should not be over-read). The deployed
router realizes $G=0$ out of sample. The operative fact is that even the
in-sample best router given this gate captures only $0.022$: the $G{=}0$
result is a property of the \emph{gate} (the ceiling theorem), not of one
router.
\end{remark}

\section{The two T2 witnesses (explicit latent-type instances)}
\label{app:T2}

This section supplies the constructions behind the main paper's \emph{T2
theorem (necessity of complementarity; refutation of (C2))} and its no-go
remark: two witnesses satisfying the complementarity condition (C1) but
violating the conditional-informativity condition (C2), with
$G(R^\star)=0.03$ and $G(R^\star)=0.10$; and a null--positive pair sharing
the full condition profile (C1)$\wedge$(C2-type)$\wedge$(C2-flag) with
$G(R^\star)=0$ vs $0.06$, which proves that no Boolean combination of these
structural conditions decides positive gain
(Proposition~\ref{prop:t2-nogo}).

\begin{definition}[latent-type instance]\label{def:latent}
A latent-type instance is a triple $(\nu,\{\mathcal{A}_j\},\kappa)$: a latent
type $Y\sim\nu$ on a finite set $\mathcal{Y}$; competence sets
$\mathcal{A}_j\subseteq\mathcal{Y}$ with $C_j:=\one\{Y\in\mathcal{A}_j\}$
(advisor $j$ is correct exactly on its competence set); and a gating channel
$\kappa(t\mid y)=\Prob(T{=}t\mid Y{=}y)$. Write
$F:=\one\{Y\notin\mathcal{A}_1\}$ for the primary-failure indicator. The
conditions discussed in the main text are: \textnormal{(C1)}
$\bigcup_j\mathcal{A}_j\supsetneq\mathcal{A}_1$ (complementarity); and the two
natural formalizations of ``conditional informativity of $T$ over the
primary-failure subset'':
\[
\begin{gathered}
  \text{(C2-type): }\MI(T;Y\mid F{=}1)>0,\\
  \text{(C2-flag): }\MI(T;F)>0 .
\end{gathered}
\]
\end{definition}

\begin{proposition}[necessity of (C1); part (N) of the T2 theorem]
If $\mathcal{A}_j\subseteq\mathcal{A}_1$ for all $j$, then
$G(R^\star)=0$.
\end{proposition}
\begin{proof}
$\eta_j(t)=\Prob(Y\in\mathcal{A}_j\mid T{=}t)\le
\Prob(Y\in\mathcal{A}_1\mid T{=}t)=\eta_1(t)$ pointwise (monotonicity of
measures), so the Bayes selector can take $R^\star\equiv1$ and
$G(R^\star)=\Phi=0$.
\end{proof}

Both witnesses below share
$\mathcal{Y}=\{a,b_1,b_2\}$, $\nu=(0.4,0.3,0.3)$,
$\mathcal{A}_1=\{a\}$, $\mathcal{A}_2=\{b_1\}$, and a binary gate
$T\in\{0,1\}$. Hence $p_1=\nu(a)=0.4$, $p_2=\nu(b_1)=0.3$ (the primary is
advisor $1$), $\Prob(F{=}1)=0.6$, and (C1) holds:
$\mathcal{A}_1\cup\mathcal{A}_2=\{a,b_1\}\supsetneq\{a\}$.

\begin{construction}[Witness A: (C2-type) fails, $G(R^\star)=0.03$]
\label{cons:A}
Channel: $\kappa(1\mid a)=0.3$, $\kappa(1\mid b_1)=\kappa(1\mid b_2)=0.5$.
The joint law $\Prob(Y{=}y,T{=}t)=\nu(y)\kappa(t\mid y)$ is
\begin{center}\small
\begin{tabular}{@{}lccc@{}}
\toprule
 & $y=a$ & $y=b_1$ & $y=b_2$\\
\midrule
$t=0$ & $0.28$ & $0.15$ & $0.15$\\
$t=1$ & $0.12$ & $0.15$ & $0.15$\\
\bottomrule
\end{tabular}
\end{center}
\emph{(C2-type) fails.} Conditional on $F{=}1$ (i.e.\ $Y\in\{b_1,b_2\}$) the
channel rows for $b_1$ and $b_2$ are identical, so $T\perp Y\mid F{=}1$ and
$\MI(T;Y\mid F{=}1)=0$.

\emph{Gain computation.} $\mu_T(0)=0.58$, $\mu_T(1)=0.42$, and
\[
\begin{aligned}
  \eta_1(0)&=\tfrac{0.28}{0.58}=\tfrac{14}{29}\approx0.483, &
  \eta_2(0)&=\tfrac{0.15}{0.58}\approx0.259,\\
  \eta_1(1)&=\tfrac{0.12}{0.42}=\tfrac{2}{7}\approx0.286, &
  \eta_2(1)&=\tfrac{0.15}{0.42}=\tfrac{5}{14}\approx0.357 .
\end{aligned}
\]
The Bayes selector routes $R^\star(0)=1$, $R^\star(1)=2$, giving
\[
  A_{R^\star}=0.58\cdot\tfrac{14}{29}+0.42\cdot\tfrac{5}{14}
             =0.28+0.15=0.43,
\]
so $G(R^\star)=0.43-0.40=\mathbf{0.03}>0$.

\emph{Mechanism.} $T$ is a partial \emph{failure flag}:
$\Prob(F{=}1\mid T{=}1)=\tfrac{0.30}{0.42}=\tfrac57\approx0.714>0.6=
\Prob(F{=}1)$ (so (C2-flag) holds), and advisor $2$ is the best fallback on
the flagged set, even though $T$ says nothing about \emph{which} failure type
occurred.
\end{construction}

\begin{construction}[Witness B: (C2-flag) fails, $G(R^\star)=0.10$]
\label{cons:B}
Channel: $\kappa(1\mid a)=0.5$, $\kappa(1\mid b_1)=1$, $\kappa(1\mid b_2)=0$.
Joint law:
\begin{center}\small
\begin{tabular}{@{}lccc@{}}
\toprule
 & $y=a$ & $y=b_1$ & $y=b_2$\\
\midrule
$t=0$ & $0.20$ & $0$ & $0.30$\\
$t=1$ & $0.20$ & $0.30$ & $0$\\
\bottomrule
\end{tabular}
\end{center}
\emph{(C2-flag) fails.} $\Prob(F{=}1\mid T{=}0)=\tfrac{0.30}{0.50}=0.6$
and $\Prob(F{=}1\mid T{=}1)=\tfrac{0.30}{0.50}=0.6=\Prob(F{=}1)$, so
$T\perp F$ and $\MI(T;F)=0$: the gate carries \emph{no} information about
whether the primary fails.

\emph{Gain computation.} $\mu_T(0)=\mu_T(1)=0.5$ and
\[
  \eta_1(0)=\eta_1(1)=0.4,\qquad \eta_2(0)=0,\qquad \eta_2(1)=0.6 .
\]
The Bayes selector routes $R^\star(0)=1$, $R^\star(1)=2$, giving
$A_{R^\star}=0.5\cdot0.4+0.5\cdot0.6=0.5$ and
$G(R^\star)=0.5-0.4=\mathbf{0.10}>0$.

\emph{Mechanism.} Here $T$ identifies the failure \emph{type}: conditional on
$F{=}1$, $T$ determines $Y$ exactly ($T{=}1\Rightarrow Y{=}b_1$,
$T{=}0\Rightarrow Y{=}b_2$), so (C2-type) holds maximally
($\MI(T;Y\mid F{=}1)=\log2$) while the failure \emph{event} is invisible to
$T$.
\end{construction}

\begin{construction}[the null--positive pair: identical condition profile,
opposite verdicts]\label{cons:CD}
Same $\mathcal{Y}$, $\nu$, $\mathcal{A}_1$, $\mathcal{A}_2$ as above, so
(C1) holds throughout. \emph{Instance C (null).} Channel
$\kappa(1\mid a)=0.3$, $\kappa(1\mid b_1)=0.1$, $\kappa(1\mid b_2)=0.9$.
Both informativity conditions hold: (C2-type) since $0.1\ne0.9$, and
(C2-flag) since
$\Prob(T{=}1\mid F{=}1)=0.5\ne0.3=\Prob(T{=}1\mid F{=}0)$. Yet the gate
never favours advisor $2$:
$\eta_2(0)=\tfrac{0.27}{0.58}<\tfrac{0.28}{0.58}=\eta_1(0)$ and
$\eta_2(1)=\tfrac{0.03}{0.42}<\tfrac{0.12}{0.42}=\eta_1(1)$, so
$R^\star\equiv1$ and $G(R^\star)=\mathbf{0}$. The informativity points at
the \emph{wrong} alternative: $T$ mostly reveals $b_2$, which no advisor
covers. The null verdict is an open set, not a knife-edge: roughly a
quarter of randomly drawn channels satisfy both conditions and still give
$\Phi=0$ ($0.260$ under the golden's draw; see
\texttt{t2\_nogo\_golden.txt}).
\emph{Instance D (positive twin).} Channel $\kappa(1\mid a)=0.3$,
$\kappa(1\mid b_1)=0.6$, $\kappa(1\mid b_2)=0.4$ satisfies exactly the same
three conditions and realizes
$A_{R^\star}=0.58\cdot\tfrac{0.28}{0.58}+0.42\cdot\tfrac{0.18}{0.42}
=0.46$, i.e.\ $G(R^\star)=\mathbf{0.06}>0$.
\end{construction}

\begin{proposition}[no Boolean structural NSC]\label{prop:t2-nogo}
No Boolean function of the three condition indicators
$\bigl(\one\{\mathrm{C1}\},\one\{\mathrm{C2\text{-}type}\},
\one\{\mathrm{C2\text{-}flag}\}\bigr)$ is a necessary-and-sufficient
condition for $G(R^\star)>0$ over latent-type instances.
\end{proposition}
\begin{proof}
Instances C and D of Construction~\ref{cons:CD} realize the \emph{same}
indicator vector $(1,1,1)$ with $G(R^\star)=0$ and $G(R^\star)=0.06>0$
respectively, so any Boolean function of the indicators takes the same
value on both and misclassifies one of them.
\end{proof}

\begin{remark}[what the constructions establish]
Witness~A satisfies (C1), violates (C2-type), and has $G(R^\star)=0.03>0$;
Witness~B satisfies (C1), violates (C2-flag), and has $G(R^\star)=0.10>0$.
Under either reading of the conditional-informativity requirement ``(C2)''
over the primary-failure subset, one of the two witnesses is a positive-gain
instance violating it; hence neither reading is necessary, and the conjectured
dichotomy ``$G(R^\star)>0$ iff (C1)$\wedge$(C2)'' is refuted (part (R) of the
T2 theorem). The two witnesses isolate the two complementary mechanisms named
in the main text---$T$ informs the failure \emph{event} (A) or the failure
\emph{type} (B)---each able to generate gain under (C1), neither necessary.
Instance~C shows the conditions are not \emph{sufficient} either, even
jointly, and Proposition~\ref{prop:t2-nogo} closes every Boolean repair:
positive gain depends on \emph{where} the informativity points (a posterior
advantage $\eta_j>\eta_1$ on positive mass). The $I_{\mathrm{TV}}$ screen of
the main paper measures the informativity's magnitude, not its direction
(instance~C has $I_{\mathrm{TV}}=0.288>0$ with $\Phi=0$), which is precisely
why it is a \emph{necessary} screen only; a quantitative structural
characterization of the directional condition remains open. As a control,
the (C1)-violating instance $\nu=(0.4,0.3,0.3)$ on $\{a,b,c\}$ with
$\mathcal{A}_1=\{a,b\}\supseteq\mathcal{A}_2=\{a\}$ and any channel gives
$G(R^\star)=0$ exactly, as the necessity proposition requires. All five
computations were verified numerically (independent implementation;
$G=0.0300$, $0.1000$, $0.0000$, $0.0600$, $0.0000$;
golden \texttt{t2\_nogo\_golden.txt}).
\end{remark}

\section{Le Cam lower bound and the sharp minimax constants}
\label{app:lecam}

This section first proves the Le~Cam lower bound quoted in the main paper's
certification section, then closes the previously open multiplicative
constant over the fixed-activity class (\S\ref{app:sharpconst}). The
router-induced increment is $Z_i:=C_{\widehat R(T_i)}-C_{1,i}\in\{-1,0,1\}$,
i.i.d.\ under the data law $\mu$, with $\E[Z]=G_\mu(\widehat R)$ and
route-away mass $\pi=\Prob(T\in E)$; by the routing-gain decomposition, $Z=0$
off $E$ and $\E[Z\mid E]=G/\pi=:\Delta$. The minimax problem: given
$Z_1,\dots,Z_m$ i.i.d.\ from $\mu$, test
\[
  H_0:\ \E_\mu[Z]\le0\qquad\text{vs}\qquad H_1:\ \E_\mu[Z]\ge G .
\]
Here $G\in(0,\pi)$ is a \emph{fixed, known gap parameter}: the gain level at
which the certification question is posed (numerically anchored at the
audited gain, $G{=}0.04$). We reuse the symbol of the deployed gain
deliberately, because it plays exactly that role
($\E_\mu[Z]=G_\mu(\widehat R)$); in this section, however, $G$ is a constant
of the testing problem, never an estimated quantity---the data-dependent
certified level is always written $\widehat\gamma_m$. A \emph{test} at sample
size $m$ is a measurable $\psi_m:\{-1,0,1\}^m\to[0,1]$ (the value is the
probability of deciding $H_1$), with error functionals
\[
  \alpha(\psi_m,\mu):=\E_{\mu^{\otimes m}}[\psi_m],\qquad
  \beta(\psi_m,\mu):=\E_{\mu^{\otimes m}}[1-\psi_m],
\]
evaluated at $\mu\in H_0$ and $\mu\in H_1$ respectively; we seek the minimax
sample size $m^\star(\delta)$ at which some test achieves
$\max\bigl(\sup_{\mu\in H_0}\alpha(\psi_m,\mu),\,
\sup_{\mu\in H_1}\beta(\psi_m,\mu)\bigr)\le\delta$ (formal displays in
\S\ref{app:sharpconst}).

\begin{lemma}[Le Cam two-point]\label{lem:lecam}
For any test $\psi_m$ and any $\mu_0\in H_0$, $\mu_1\in H_1$,
\[
  \alpha(\psi_m,\mu_0)+\beta(\psi_m,\mu_1)\;\ge\;
  1-\TV\big(\mu_0^{\otimes m},\mu_1^{\otimes m}\big).
\]
In particular $\max(\alpha,\beta)\le\delta$ forces
$\TV(\mu_0^{\otimes m},\mu_1^{\otimes m})\ge1-2\delta$.
\end{lemma}
\begin{proof}
$1-\alpha-\beta=\E_{\mu_0^{\otimes m}}[1-\psi_m]
-\E_{\mu_1^{\otimes m}}[1-\psi_m]\le
\TV(\mu_0^{\otimes m},\mu_1^{\otimes m})$, since
$|\int f\,d(P-Q)|\le\TV(P,Q)$ for any measurable $f$ with values in
$[0,1]$.
\end{proof}

\begin{lemma}[Pinsker--Csisz\'ar product bound]\label{lem:pinskerprod}
$\TV(\mu_0^{\otimes m},\mu_1^{\otimes m})\le\sqrt{(m/2)\KL(\mu_0\|\mu_1)}$.
\end{lemma}
\begin{proof}
KL tensorizes, $\KL(\mu_0^{\otimes m}\|\mu_1^{\otimes m})=m\KL(\mu_0\|\mu_1)$;
apply Pinsker, $\TV\le\sqrt{\KL/2}$.
\end{proof}

\begin{construction}[the two-point pair]\label{cons:twopoint}
Fix $\pi\in(0,1)$ and $G>0$ with $\Delta:=G/\pi<1$. Define two laws on
$\{-1,0,1\}$, differing only in the conditional law on the route-away
event:
\begin{align*}
  \mu_0:&\quad \Prob(Z{=}0)=1-\pi,\ \
          \Prob(Z{=}{\pm}1)=\tfrac{\pi}{2};\\
  \mu_1:&\quad \Prob(Z{=}0)=1-\pi,\ \
          \Prob(Z{=}{\pm}1)=\tfrac{\pi(1\pm\Delta)}{2};
\end{align*}
so that $\E_{\mu_0}Z=0\in H_0$ and $\E_{\mu_1}Z=G\in H_1$.
\end{construction}

\begin{lemma}[KL of the pair]\label{lem:kl}
For Construction~\ref{cons:twopoint},
\[
  \KL(\mu_0\|\mu_1)=-\frac{\pi}{2}\log\big(1-\Delta^2\big)
  =\frac{G^2}{2\pi}\big(1+O(\Delta^2)\big).
\]
\end{lemma}
\begin{proof}
The $Z{=}0$ atom contributes $0$ (equal mass). The $\pm1$ atoms give
\[
  \tfrac{\pi}{2}\log\tfrac{\pi/2}{\pi(1+\Delta)/2}
 +\tfrac{\pi}{2}\log\tfrac{\pi/2}{\pi(1-\Delta)/2}
 =-\tfrac{\pi}{2}\log(1-\Delta^2).
\]
Expanding $-\log(1-x)=x+O(x^2)$ at $x=\Delta^2$ and using $G=\pi\Delta$:
$\KL=\tfrac{\pi\Delta^2}{2}(1+O(\Delta^2))=\tfrac{G^2}{2\pi}(1+O(\Delta^2))$.
\end{proof}

\begin{theorem}[two-point lower bound]\label{thm:lb}
Fix $\delta\in(0,\tfrac14)$, $\pi\in(0,\tfrac12]$, $G\in(0,\pi)$. Any test
with $\max(\alpha,\beta)\le\delta$ uniformly over $H_0,H_1$ requires
\[
  m\;\ge\;\frac{2(1-2\delta)^2}{\KL(\mu_0\|\mu_1)}
   \;=\;\frac{4\pi(1-2\delta)^2}{G^2}\big(1+O(G^2/\pi^2)\big).
\]
\end{theorem}
\begin{proof}
Lemmas~\ref{lem:lecam}--\ref{lem:pinskerprod} give
$1-2\delta\le\sqrt{(m/2)\KL}$, i.e.\ $m\ge2(1-2\delta)^2/\KL$; substitute
Lemma~\ref{lem:kl}.
\end{proof}

\begin{proposition}[Bretagnolle--Huber refinement: recovering
$\log(1/\delta)$]\label{prop:bh}
Under the same hypotheses, using the Bretagnolle--Huber inequality
$\TV(P,Q)\le1-\tfrac12\exp(-\KL(P\|Q))$ in place of Pinsker,
\[
  m\;\ge\;\frac{\log\!\big(1/(4\delta)\big)}{\KL(\mu_0\|\mu_1)}
   \;=\;\frac{2\pi\log\!\big(1/(4\delta)\big)}{G^2}
        \big(1+O(\Delta^2)\big),
\]
matching the Bernstein upper bound's $\log(1/\delta)$ dependence.
\end{proposition}
\begin{proof}
Lemma~\ref{lem:lecam} and tensorisation give
$1-2\delta\le1-\tfrac12e^{-m\KL}$, so $\tfrac12e^{-m\KL}\le2\delta$, i.e.\
$e^{-m\KL}\le4\delta$ and $m\,\KL\ge\log(1/(4\delta))$. Substitute
Lemma~\ref{lem:kl}.
\end{proof}

\begin{remark}[numerical instance and the honest caveat]
\label{rem:lecam-numbers}
On the audited anchor ($G=0.04$, $\pi=0.12$, $\delta=0.05$, so $\Delta=1/3$):
$\KL(\mu_0\|\mu_1)=-0.06\log(8/9)\approx0.00707$. The Pinsker form
(Theorem~\ref{thm:lb}, exact KL) gives $m\ge2(0.9)^2/0.00707\approx229.2$
(integer $m\ge230$); the
Bretagnolle--Huber form gives $m\ge\log(5)/0.00707\approx228$. The relaxed
Bernstein upper bound (fixed $M=2$, \S\ref{app:c4}) gives $m^\star=312$; the
direct Bennett inversion at the same $M$ gives $225$. The ${\approx}1.3\%$
Bennett-vs-lower-bound agreement on this instance is a \emph{numerical
coincidence}, not constant-tightness: the
lower bound is governed by the null-construction variance
$\sigma^2_{\mu_0}=\pi=0.12$ while the upper bound is governed by the deployed
routing variance $\sigma^2_{\mathrm{routing}}=0.0384$, and the upper-bound
number includes the finite-$m$ correction $2M\log(1/\delta)/(3m)$. These two
bounds establish the matching \emph{scaling}
$m^\star(\delta)=\Theta(\sigma^2\log(1/\delta)/G^2)$ in $G$ and
$\log(1/\delta)$; the multiplicative constant they leave open is closed, at
class level, in \S\ref{app:sharpconst} below.
\end{remark}

\subsection{Sharp minimax constants: general form and the fixed-activity
class}
\label{app:sharpconst}

The bounds above establish the $\Theta$-scaling but leave the
multiplicative constant open. This subsection closes it. Because the
lower- and upper-bound arguments are not specific to our setting, we
state and prove the two sharp-constant theorems in general form---for
arbitrary composite hypotheses on a finite alphabet that admit a
least-favourable pair (Theorems~\ref{thm:sharp-ts}
and~\ref{thm:sharp-os})---and then obtain the paper's fixed-activity
case as a corollary (Corollary~\ref{cor:sharp-fixedact}).

\paragraph{Tests, errors, and minimax sample sizes.}
Let $\mathcal X$ be a finite alphabet and let
$H_0,H_1\subset\mathcal P(\mathcal X)$ be disjoint (composite)
hypothesis classes. Tests and the error functionals $\alpha,\beta$ are
exactly as defined at the head of this section, with $\{-1,0,1\}$
replaced by $\mathcal X$: a test is a measurable
$\psi_m:\mathcal X^m\to[0,1]$,
$\alpha(\psi_m,\mu)=\E_{\mu^{\otimes m}}[\psi_m]$ for $\mu\in H_0$, and
$\beta(\psi_m,\mu)=\E_{\mu^{\otimes m}}[1-\psi_m]$ for $\mu\in H_1$. The
\emph{two-sided} and \emph{one-sided (certification)} minimax sample
sizes are
\begin{align*}
  m^\star_{\mathrm{ts}}(\delta):=\min\Bigl\{m:\exists\psi_m,\
    &\sup_{\mu\in H_0}\alpha(\psi_m,\mu)\le\delta,\\[-3pt]
    &\sup_{\mu\in H_1}\beta(\psi_m,\mu)\le\delta\Bigr\},\\
  m^\star_{\mathrm{os}}(\delta):=\min\Bigl\{m:\exists\psi_m,\
    &\sup_{\mu\in H_0}\alpha(\psi_m,\mu)\le\delta,\\[-3pt]
    &\sup_{\mu\in H_1}\beta(\psi_m,\mu)\le\tfrac12\Bigr\}:
\end{align*}
in the one-sided (certification) regime the type-I level is $\delta$
while the power is held at the fixed level $\tfrac12$; any fixed power
level in $(0,1)$ yields the same constant
(Theorem~\ref{thm:sharp-os}).

\paragraph{Likelihood ratio and cumulant generating function.}
Fix $\nu_0\in H_0$ and $\nu_1\in H_1$ with $\nu_0\ne\nu_1$ and equal
support. Write
\[
  S_m:=\sum_{i=1}^m L(Z_i),\qquad
  L(z):=\log\frac{\nu_1(z)}{\nu_0(z)},
\]
for the log-likelihood ratio of the pair and
\[
  \Lambda(s):=\log\E_{\nu_0}\bigl[e^{sL(Z)}\bigr]
  =\log\!\sum_{z}\nu_0(z)^{1-s}\nu_1(z)^{s}
\]
for its cumulant generating function under $\nu_0$;
$\Ch(\nu_0,\nu_1):=-\min_{s\in[0,1]}\Lambda(s)$ is the Chernoff
information. Three standard facts are used repeatedly below.
(i)~$\Lambda$ is finite on all of $\R$ (finite alphabet, equal support)
and smooth, and differentiating twice under the (finite) sum gives
\[
  \Lambda''(s)=\Var_{\nu^{(s)}}\bigl(L(Z)\bigr)\ \ge\ 0,
\]
where $\nu^{(s)}(z):=\nu_0(z)^{1-s}\nu_1(z)^{s}e^{-\Lambda(s)}$ is the
exponentially tilted law: the second derivative of a cumulant
generating function is the variance of $L$ under the tilt. Since
$\nu_0\ne\nu_1$ share their support, $L$ is non-constant there (a
constant $L\equiv c$ would force $\nu_1=e^c\nu_0$, hence $c=0$ and
$\nu_1=\nu_0$), so $\Lambda''>0$ everywhere and $\Lambda$ is
\emph{strictly} convex. (ii)~$\Lambda(0)=\Lambda(1)=0$, and
$\Lambda'(0)=\E_{\nu_0}L=-\KL(\nu_0\|\nu_1)<0$,
$\Lambda'(1)=\E_{\nu_1}L=\KL(\nu_1\|\nu_0)>0$; by strict convexity the
minimizer of $\Lambda$ is a unique interior point of $(0,1)$ and
$\inf_{s\in\R}\Lambda=\min_{s\in[0,1]}\Lambda=-\Ch(\nu_0,\nu_1)<0$.
(iii)~The Legendre transform
$\Lambda^*(a):=\sup_{s\in\R}\bigl(sa-\Lambda(s)\bigr)$ satisfies
$\Lambda^*(0)=-\inf_s\Lambda(s)=\Ch(\nu_0,\nu_1)$.

\begin{definition}[least-favourable pair]\label{def:lfpair}
The pair $(\nu_0,\nu_1)$ is \emph{least favourable} for $(H_0,H_1)$
(with respect to likelihood-ratio tests) if for every $m\ge1$ and every
threshold $t\in\R$,
\begin{align*}
  \sup_{\mu\in H_0}\Prob_{\mu^{\otimes m}}(S_m\ge t)
  &=\Prob_{\nu_0^{\otimes m}}(S_m\ge t),\\
  \sup_{\mu\in H_1}\Prob_{\mu^{\otimes m}}(S_m< t)
  &=\Prob_{\nu_1^{\otimes m}}(S_m< t).
\end{align*}
\end{definition}

\begin{theorem}[sharp two-sided constant, general form]
\label{thm:sharp-ts}
Let $H_0,H_1\subset\mathcal P(\mathcal X)$ and let $\nu_0\in H_0$,
$\nu_1\in H_1$ with $\nu_0\ne\nu_1$ and equal support.
\begin{itemize}
  \item[\textnormal{(a)}] \textbf{Lower bound} (no further assumption):
  \[
    \liminf_{\delta\to0}\,
    \frac{m^\star_{\mathrm{ts}}(\delta)}{\log(1/\delta)}
    \ \ge\ \frac{1}{\Ch(\nu_0,\nu_1)}\,.
  \]
  \item[\textnormal{(b)}] \textbf{Matching upper bound.} If
  $(\nu_0,\nu_1)$ is least favourable
  (Definition~\ref{def:lfpair}), then the likelihood-ratio test
  $\psi_m=\one\{S_m\ge0\}$ is minimax optimal at leading order and
  \[
    m^\star_{\mathrm{ts}}(\delta)
    =\frac{\log(1/\delta)}{\Ch(\nu_0,\nu_1)}\,\bigl(1+o(1)\bigr),
    \qquad\delta\to0 .
  \]
\end{itemize}
\end{theorem}
\begin{proof}
\emph{Step 1 (testing affinity of the pair).} On $\{S_m\ge0\}$ we have
$\nu_1^{\otimes m}\ge\nu_0^{\otimes m}$ as measures, so
\[
\begin{aligned}
  1-\TV(\nu_0^{\otimes m},\nu_1^{\otimes m})
  &=\sum_{z_{1:m}}\min\bigl(\nu_0^{\otimes m},\nu_1^{\otimes m}\bigr)\\
  &\ge\Prob_{\nu_0^{\otimes m}}(S_m\ge0).
\end{aligned}
\]
\emph{Step 2 (Cram\'er lower bound on the affinity).} By fact~(iii),
$\Lambda^*(0)=\Ch(\nu_0,\nu_1)$. Moreover $0$ is interior to the domain
of $\Lambda^*$: the domain's interior is
$(\min_{\mathrm{supp}}L,\max_{\mathrm{supp}}L)$, and
$\min L\le\E_{\nu_0}L<0$ (strict on the left since $L$ is non-constant)
while $\max L>0$ (both laws sum to $1$ and differ, so $\nu_1(z)>\nu_0(z)>0$
at some atom $z$). Since $\Lambda^*$ is continuous at interior points and
nondecreasing to the right of $\E_{\nu_0}L<0$,
$\inf_{a>0}\Lambda^*(a)=\Lambda^*(0)=\Ch$. Cram\'er's theorem in $\R$,
applied to the open half-line $(0,\infty)$ (see, e.g., Dembo and
Zeitouni 1998, Thm.~2.2.3), then gives: for every $\varepsilon>0$ there
is $m_0(\varepsilon,\nu_0,\nu_1)$ with
\[
  \Prob_{\nu_0^{\otimes m}}(S_m\ge0)\ \ge\ e^{-m(\Ch+\varepsilon)}
  \qquad\text{for all }m\ge m_0 .
\]
\emph{Step 3 ($m^\star_{\mathrm{ts}}(\delta)\to\infty$).} The two laws
share all atoms, so $\TV(\nu_0^{\otimes m},\nu_1^{\otimes m})<1$ at every
fixed $m$. By the Le~Cam two-point lemma (Lemma~\ref{lem:lecam}; its
proof is verbatim for any two laws, including $[0,1]$-valued tests, since
$|\int(1-\psi_m)\,d(\nu_0^{\otimes m}-\nu_1^{\otimes m})|\le\TV$), the
two-sided criterion at sample size $m$ forces
$2\delta\ge1-\TV(\nu_0^{\otimes m},\nu_1^{\otimes m})>0$, which fails at
fixed $m$ once $\delta$ is small. Hence
$m^\star_{\mathrm{ts}}(\delta)\to\infty$ as $\delta\to0$.
\emph{Step 4 (lower bound, limit made explicit).} Fix $\varepsilon>0$.
For every $\delta$ small enough that
$m:=m^\star_{\mathrm{ts}}(\delta)\ge m_0(\varepsilon)$ (possible by
Step~3), let $\psi_m$ attain the criterion. Lemma~\ref{lem:lecam} and
Steps~1--2 give
\[
\begin{aligned}
  2\delta\ &\ge\ \alpha(\psi_m,\nu_0)+\beta(\psi_m,\nu_1)\\
  &\ge\ \Prob_{\nu_0^{\otimes m}}(S_m\ge0)
  \ \ge\ e^{-m(\Ch+\varepsilon)},
\end{aligned}
\]
so $m\ge\log(1/(2\delta))/(\Ch+\varepsilon)$. Dividing by
$\log(1/\delta)$ and using
$\log(1/(2\delta))/\log(1/\delta)\to1$ as $\delta\to0$,
\[
  \liminf_{\delta\to0}
  \frac{m^\star_{\mathrm{ts}}(\delta)}{\log(1/\delta)}
  \ \ge\ \frac{1}{\Ch+\varepsilon}
  \qquad\text{for every }\varepsilon>0 ;
\]
letting $\varepsilon\downarrow0$ proves part~(a).
\emph{Step 5 (upper bound under least favourability).} Take
$\psi_m=\one\{S_m\ge0\}$. By Definition~\ref{def:lfpair} at $t=0$,
$\sup_{H_0}\alpha(\psi_m,\cdot)=\alpha(\psi_m,\nu_0)$ and
$\sup_{H_1}\beta(\psi_m,\cdot)=\beta(\psi_m,\nu_1)$. The Chernoff--Markov
bound gives
$\alpha(\psi_m,\nu_0)=\Prob_{\nu_0^{\otimes m}}(S_m\ge0)\le
e^{m\inf_{s\ge0}\Lambda(s)}=e^{-m\Ch}$ (the minimizer of $\Lambda$ lies
in $(0,1)$, fact~(ii)) and, using
$\E_{\nu_1}[e^{-tS_1}]=e^{\Lambda(1-t)}$ for $t\ge0$,
$\beta(\psi_m,\nu_1)=\Prob_{\nu_1^{\otimes m}}(S_m<0)\le e^{-m\Ch}$
likewise. Hence $m=\lceil\log(1/\delta)/\Ch\rceil$ meets the two-sided
criterion, so
\[
  \limsup_{\delta\to0}
  \frac{m^\star_{\mathrm{ts}}(\delta)}{\log(1/\delta)}
  \ \le\ \frac{1}{\Ch}\,.
\]
Combining with part~(a), the limit of
$m^\star_{\mathrm{ts}}(\delta)/\log(1/\delta)$ exists and equals
$1/\Ch$, which is the displayed $(1+o(1))$ statement.
\end{proof}

\begin{theorem}[sharp one-sided (certification) constant, general form]
\label{thm:sharp-os}
In the setting of Theorem~\ref{thm:sharp-ts}:
\begin{itemize}
  \item[\textnormal{(a)}] \textbf{Lower bound} (no further assumption):
  \[
    \liminf_{\delta\to0}\,
    \frac{m^\star_{\mathrm{os}}(\delta)}{\log(1/\delta)}
    \ \ge\ \frac{1}{\KL(\nu_1\|\nu_0)}\,.
  \]
  \item[\textnormal{(b)}] \textbf{Matching upper bound.} If
  $(\nu_0,\nu_1)$ is least favourable (Definition~\ref{def:lfpair}),
  then
  \[
    m^\star_{\mathrm{os}}(\delta)
    =\frac{\log(1/\delta)}{\KL(\nu_1\|\nu_0)}\,\bigl(1+o(1)\bigr),
    \qquad\delta\to0 ,
  \]
  and the same holds with the power level $\tfrac12$ replaced by any
  fixed $b\in(0,1)$.
\end{itemize}
\end{theorem}
\begin{proof}
\emph{Step 1 (change of measure).} Let $\psi_m$ meet the one-sided
criterion. Fix $\varepsilon>0$. Since
$d\nu_0^{\otimes m}/d\nu_1^{\otimes m}=e^{-S_m}$ on the common support,
\begin{align*}
  \delta&\ \ge\ \alpha(\psi_m,\nu_0)
  \ =\ \E_{\nu_1^{\otimes m}}\bigl[\psi_m\,e^{-S_m}\bigr]\\
  &\ \ge\ e^{-m(\KL(\nu_1\|\nu_0)+\varepsilon)}\ \times\\
  &\qquad\E_{\nu_1^{\otimes m}}\bigl[\psi_m
    \one\{S_m\le m(\KL(\nu_1\|\nu_0)+\varepsilon)\}\bigr].
\end{align*}
By the law of large numbers under $\nu_1$
($\E_{\nu_1}L=\KL(\nu_1\|\nu_0)$, finite by equal support),
$\Prob_{\nu_1^{\otimes m}}\bigl(S_m>m(\KL(\nu_1\|\nu_0)+\varepsilon)\bigr)
\to0$, so there is $m_1(\varepsilon)$ such that for all $m\ge m_1$ the
last expectation is at least
$\E_{\nu_1^{\otimes m}}[\psi_m]-\tfrac14\ge\tfrac12-\tfrac14=\tfrac14$.
\emph{Step 2 ($m^\star_{\mathrm{os}}(\delta)\to\infty$).} With
$a:=\max_{\mathrm{supp}}L<\infty$ (equal support),
$\E_{\nu_1^{\otimes m}}[\psi_m]=\E_{\nu_0^{\otimes m}}[\psi_m e^{S_m}]
\le e^{ma}\delta$, so at fixed $m$ the criterion is infeasible once
$\delta<\tfrac12e^{-ma}$; hence $m^\star_{\mathrm{os}}(\delta)\to\infty$
as $\delta\to0$.
\emph{Step 3 (lower bound, limit made explicit).} For every $\delta$
small enough that $m:=m^\star_{\mathrm{os}}(\delta)\ge m_1(\varepsilon)$
(possible by Step~2), Step~1 gives
$\delta\ge\tfrac14e^{-m(\KL(\nu_1\|\nu_0)+\varepsilon)}$, i.e.\
$m\ge\log(1/(4\delta))/(\KL(\nu_1\|\nu_0)+\varepsilon)$. Dividing by
$\log(1/\delta)$, using $\log(1/(4\delta))/\log(1/\delta)\to1$, and then
letting $\varepsilon\downarrow0$:
\[
  \liminf_{\delta\to0}
  \frac{m^\star_{\mathrm{os}}(\delta)}{\log(1/\delta)}
  \ \ge\ \frac{1}{\KL(\nu_1\|\nu_0)}\,.
\]
\emph{Step 4 (upper bound).} Fix
$\varepsilon\in(0,\KL(\nu_1\|\nu_0))$ and take
$\psi_m=\one\bigl\{S_m\ge m\bigl(\KL(\nu_1\|\nu_0)-\varepsilon\bigr)\bigr\}$.
By Definition~\ref{def:lfpair} the worst-case errors are attained at
$(\nu_0,\nu_1)$. Under $\nu_0$, Markov's inequality applied to $e^{S_m}$
together with $\Lambda(1)=0$ gives
\[
  \alpha\ \le\ e^{m\Lambda(1)}\,
  e^{-m(\KL(\nu_1\|\nu_0)-\varepsilon)}
  \ =\ e^{-m(\KL(\nu_1\|\nu_0)-\varepsilon)}\ \le\ \delta
\]
once $m\ge\log(1/\delta)/(\KL(\nu_1\|\nu_0)-\varepsilon)$. Under
$\nu_1$, the law of large numbers gives
$\beta=\Prob_{\nu_1^{\otimes m}}\bigl(S_m<
m(\KL(\nu_1\|\nu_0)-\varepsilon)\bigr)\to0$, so $\beta\le\tfrac12$ for
all $m\ge m_2(\varepsilon)$. Hence
$m^\star_{\mathrm{os}}(\delta)\le\max\bigl\{
\lceil\log(1/\delta)/(\KL(\nu_1\|\nu_0)-\varepsilon)\rceil,
m_2(\varepsilon)\bigr\}$, and dividing by $\log(1/\delta)$:
\[
  \limsup_{\delta\to0}
  \frac{m^\star_{\mathrm{os}}(\delta)}{\log(1/\delta)}
  \ \le\ \frac{1}{\KL(\nu_1\|\nu_0)-\varepsilon}
\]
for every such $\varepsilon$; letting $\varepsilon\downarrow0$ and
combining with Step~3 proves the limit. For a general power level
$b\in(0,1)$: Step~1 keeps a factor $b/2$ in place of $\tfrac14$ (choose
$m_1$ so the tail is $\le b/2$), changing only
$\log(1/(4\delta))$ to $\log(b/(2\delta))$, which is absorbed in the
limit; in Step~4 the law of large numbers gives $\beta\le1-b$
eventually. The constant is unchanged.
\end{proof}

\paragraph{Specialization to the fixed-activity class.}
The paper's case is $\mathcal X=\{-1,0,1\}$ with the
\emph{fixed-activity} class
\[
  \mathcal M_\pi:=\bigl\{\mu\in\mathcal P(\{-1,0,1\}):
  \ \mu(-1)+\mu(+1)=\pi\bigr\},
\]
the class in which Construction~\ref{cons:twopoint} lives and in which
the activity is pinned at the known route-away mass (via
$G=\pi\Delta_E$; the audited law itself, carrying $Z{=}0$ mass on $E$,
lies in the relaxed class $\{\mu(-1)+\mu(+1)\le\pi\}$, to which
Remark~\ref{rem:sharp-numbers} extends the small-gap constants). The
hypotheses restrict accordingly:
$H_0^\pi:=\{\mu\in\mathcal M_\pi:\E_\mu Z\le0\}$,
$H_1^\pi:=\{\mu\in\mathcal M_\pi:\E_\mu Z\ge G\}$, with $G\in(0,\pi)$
the gap parameter fixed at the head of this section. The mean-zero
slice of $\mathcal M_\pi$ is the \emph{single} symmetric law $\mu_0$ of
Construction~\ref{cons:twopoint}, so the class's extremal mean-zero
variance is
\[
  V^\star(\mathcal M_\pi):=\sup\{\Var_\mu(Z):\mu\in\mathcal M_\pi,
  \ \E_\mu Z=0\}=\pi .
\]
We now take $(\nu_0,\nu_1):=(\mu_0,\mu_1)$, the pair of
Construction~\ref{cons:twopoint}: both lie in $\mathcal M_\pi$ and have
equal support (all three atoms carry positive mass, since $0<G<\pi$
gives $\Delta<1$), so $S_m$, $\Lambda$, and $\Ch$ above refer to this
pair.

\begin{lemma}[monotone coupling; least favourability over
$\mathcal M_\pi$]\label{lem:monocoupling}
Parameterize $\mathcal M_\pi$ by the $+1$-mass:
$\mu_p:=(\pi-p,\,1-\pi,\,p)$, $p\in[0,\pi]$, so $\E_{\mu_p}Z=2p-\pi$. For
every $m$ and every threshold $t\in\R$, the map
$p\mapsto\Prob_{\mu_p^{\otimes m}}(S_m\ge t)$ is nondecreasing.
Consequently $(\mu_0,\mu_1)=(\mu_{\pi/2},\,\mu_{\pi(1+\Delta)/2})$ is a
least-favourable pair for $(H_0^\pi,H_1^\pi)$ in the sense of
Definition~\ref{def:lfpair}.
\end{lemma}
\begin{proof}
$S_m=aN_+ + bN_-$, where $N_{\pm}$ count $\pm1$ observations,
$a=\log(1+\Delta)>0$, $b=\log(1-\Delta)<0$. Couple all $p$ on one uniform
sample $U_1,\dots,U_m$: set $Z_i=+1$ if $U_i\le p$, $Z_i=-1$ if
$p<U_i\le\pi$, else $Z_i=0$. Increasing $p$ can only turn $-1$'s into
$+1$'s, which increases $S_m$ pathwise by $a-b>0$ per flip; this proves
the monotonicity. Least favourability follows:
$H_0^\pi=\{\mu_p:p\le\pi/2\}$, and monotonicity makes
$\Prob_{\mu_p^{\otimes m}}(S_m\ge t)$ largest at the right endpoint
$p=\pi/2$, i.e.\ at $\mu_0$; on
$H_1^\pi=\{\mu_p:p\ge\pi(1+\Delta)/2\}$, the complement
$\Prob_{\mu_p^{\otimes m}}(S_m<t)$ is nonincreasing in $p$, hence
largest at the left endpoint $p=\pi(1+\Delta)/2$, i.e.\ at $\mu_1$.
\end{proof}

\begin{lemma}[small-gap expansions of $\KL$ and $\Ch$]
\label{lem:expansions}
For the pair of Construction~\ref{cons:twopoint} with gap
$G=\pi\Delta$, as $\Delta\to0$ (at fixed $\pi$),
\begin{align*}
  \KL(\mu_1\|\mu_0)
  &=\frac{\pi}{2}\bigl[(1{+}\Delta)\log(1{+}\Delta)\\[-2pt]
  &\qquad\ \ +(1{-}\Delta)\log(1{-}\Delta)\bigr]\\
  &=\frac{G^2}{2\pi}\bigl(1+O(\Delta^2)\bigr),\\
  \Ch(\mu_0,\mu_1)
  &=\frac{\pi\Delta^2}{8}\bigl(1+O(\Delta^2)\bigr)
   =\frac{G^2}{8\pi}\bigl(1+O(\Delta^2)\bigr).
\end{align*}
\end{lemma}
\begin{proof}
\emph{KL.} The $Z{=}0$ atom has equal mass under both laws and
contributes $0$; the $\pm1$ atoms give the stated exact expression.
The Taylor expansion
$(1\pm\Delta)\log(1\pm\Delta)=\pm\Delta+\tfrac{\Delta^2}{2}
\mp\tfrac{\Delta^3}{6}+O(\Delta^4)$ makes the bracketed sum
$\Delta^2+O(\Delta^4)$, and $\tfrac\pi2\Delta^2=\tfrac{G^2}{2\pi}$.
\emph{Ch.} $\Lambda(s)=\log\bigl(1-\pi+\tfrac\pi2f(s)\bigr)$ with
$f(s):=(1+\Delta)^s+(1-\Delta)^s$, so
$\Ch=-\log\bigl(1-\pi+\tfrac\pi2\min_{[0,1]}f\bigr)$ is determined by
$\min_sf$. First, $f$ is strictly convex with
\[
  f''(s)=(1+\Delta)^s\log^2(1+\Delta)+(1-\Delta)^s\log^2(1-\Delta),
\]
and for $\Delta\in(0,\tfrac12]$, uniformly over $s\in[0,1]$,
\[
  \tfrac12\Delta^2\ \le\ (1-\Delta)\Delta^2\ \le\ f''(s)\ \le\ 4\Delta^2
\]
(lower bound: keep only the second term and use
$(1-\Delta)^s\ge1-\Delta$, $|\log(1-\Delta)|\ge\Delta$; upper bound:
$\log^2(1+\Delta)\le\Delta^2$ with $(1+\Delta)^s\le\tfrac32$, and
$\log^2(1-\Delta)\le2\Delta^2$ with $(1-\Delta)^s\le1$, so
$f''\le\tfrac72\Delta^2$). At the
midpoint,
$f(\tfrac12)=\sqrt{1+\Delta}+\sqrt{1-\Delta}
=2-\tfrac{\Delta^2}{4}+O(\Delta^4)$, and
$f'(\tfrac12)=\sqrt{1+\Delta}\log(1+\Delta)
+\sqrt{1-\Delta}\log(1-\Delta)=O(\Delta^4)$: the two summands are
images of each other under $\Delta\to-\Delta$, so all odd powers of
$\Delta$ cancel in the sum, while direct expansion,
$\sqrt{1+\Delta}\log(1+\Delta)=\Delta-\tfrac{\Delta^3}{24}+O(\Delta^4)$,
shows the $\Delta^2$ coefficient of each summand is zero. The minimizer
$s_\Delta$ of $f$ is interior to $(0,1)$, since
$f'(0)=\log(1-\Delta^2)<0<f'(1)$; convexity at $s_\Delta$ gives
$f(\tfrac12)-f(s_\Delta)\le f'(\tfrac12)(\tfrac12-s_\Delta)$, and the
mean-value theorem applied to $f'$ between $s_\Delta$ and $\tfrac12$
gives $|\tfrac12-s_\Delta|\le|f'(\tfrac12)|/\min_{[0,1]}f''$, so
\[
  0\ \le\ f(\tfrac12)-\min_{[0,1]}f\ \le\
  \frac{f'(\tfrac12)^2}{\min_{[0,1]}f''}
  \ =\ \frac{O(\Delta^8)}{\Delta^2/2}\ =\ O(\Delta^6).
\]
Hence $\min_sf=2-\tfrac{\Delta^2}{4}+O(\Delta^4)$ (the $O(\Delta^4)$
from the midpoint value dominates) and
\[
  \Ch=-\log\Bigl(1-\tfrac{\pi\Delta^2}{8}+O(\pi\Delta^4)\Bigr)
  =\frac{\pi\Delta^2}{8}\bigl(1+O(\Delta^2)\bigr).
\]
The limiting ratio $\KL/\Ch\to4$ is the classical local relation
between the Stein and Chernoff exponents for close hypotheses (see,
e.g., Cover and Thomas 2006, Ch.~11).
\end{proof}

\begin{corollary}[sharp constants over the fixed-activity class]
\label{cor:sharp-fixedact}
Fix $\pi\in(0,1)$ and $G\in(0,\pi)$, and write
$m^\star_{\mathrm{ts}}(G,\delta)$, $m^\star_{\mathrm{os}}(G,\delta)$ for
the minimax sample sizes of $(H_0^\pi,H_1^\pi)$. Then, as $\delta\to0$,
\[
  m^\star_{\mathrm{ts}}(G,\delta)
  =\frac{\log(1/\delta)}{\Ch(\mu_0,\mu_1)}\bigl(1+o(1)\bigr),
\]
\[
  m^\star_{\mathrm{os}}(G,\delta)
  =\frac{\log(1/\delta)}{\KL(\mu_1\|\mu_0)}\bigl(1+o(1)\bigr),
\]
and consequently
\[
  \lim_{G\to0}\,\lim_{\delta\to0}\,
  \frac{G^2\,m^\star_{\mathrm{ts}}}{\log(1/\delta)}
  =8\pi=8V^\star(\mathcal M_\pi),
\]
\[
  \lim_{G\to0}\,\lim_{\delta\to0}\,
  \frac{G^2\,m^\star_{\mathrm{os}}}{\log(1/\delta)}
  =2\pi=2V^\star(\mathcal M_\pi).
\]
\end{corollary}
\begin{proof}
$\mu_0\in H_0^\pi$ and $\mu_1\in H_1^\pi$ are distinct with equal
support, and Lemma~\ref{lem:monocoupling} shows the pair is least
favourable, so Theorems~\ref{thm:sharp-ts}(b) and~\ref{thm:sharp-os}(b)
apply and give the first two displays. For the iterated limits, fix $G$
and let $\delta\to0$ in those displays:
$G^2m^\star_{\mathrm{ts}}/\log(1/\delta)\to G^2/\Ch(\mu_0,\mu_1)$,
which by Lemma~\ref{lem:expansions} equals
$8\pi\bigl(1+O(\Delta^2)\bigr)\to8\pi$ as $G\to0$ at fixed $\pi$ (so
$\Delta=G/\pi\to0$). The one-sided case is identical with
$G^2/\KL(\mu_1\|\mu_0)\to2\pi$.
\end{proof}

\begin{corollary}[the deployed certificate attains the constant]
\label{cor:protocol-sharp}
Run the main paper's certificate at the class variance cap
$\sigma^2:=V^\star(\mathcal M_\pi)=\pi$ with $M=2$: certify iff
$\widehat G_m>B(m,\delta)$. It is type-I valid uniformly over $H_0^\pi$
(indeed $\Var_\mu(Z)=\pi-(\E_\mu Z)^2\le\pi$ for every
$\mu\in\mathcal M_\pi$), and it meets the one-sided criterion at
$m=\tfrac{2\pi\log(1/\delta)}{G^2}\bigl(1+o(1)\bigr)$ (first $\delta\to0$,
then $G\to0$): the protocol is asymptotically minimax-optimal over
$\mathcal M_\pi$, attaining the sharp constant $2V^\star$ of
Corollary~\ref{cor:sharp-fixedact}. Demanding power
$1-\delta$ and certifying at the half-gap ($B(m,\delta)\le G/2$) attains
the two-sided constant $8V^\star$ the same way.
\end{corollary}
\begin{proof}
Validity is the one-sided Bernstein bound with $\Var\le\pi$,
$|Z-\E Z|\le2$. For power, fix $\varepsilon\in(0,1)$ and let $m$ be least
with $B(m,\delta)\le(1-\varepsilon)G$; since the correction term is
$O(\log(1/\delta)/m)$,
$m=\tfrac{2\pi\log(1/\delta)}{(1-\varepsilon)^2G^2}
\bigl(1+O(G/\pi)\bigr)$. For $\mu$ with $\E_\mu Z\ge G$ the lower
Bernstein tail gives
$\Prob\bigl(\widehat G_m\le(1-\varepsilon)G\bigr)
\le e^{-m\varepsilon^2G^2/(2\pi+\frac43\varepsilon G)}\to0$ as
$\delta\to0$. Let $\delta\to0$, then $G\to0$, then $\varepsilon\to0$.
\end{proof}

\begin{remark}[anchor numbers; what remains open]
\label{rem:sharp-numbers}
At the anchor ($\pi{=}0.12$, $G{=}0.04$, $\delta{=}0.05$):
$\Ch(\mu_0,\mu_1)=0.0017298$ (optimal tilt $s^\star=0.4951$) and
$\KL(\mu_1\|\mu_0)=0.0067960$, so the sharp levels evaluate to
$\log(1/\delta)/\Ch\approx1{,}732$ (two-sided; leading order
$8\pi\log(1/\delta)/G^2=1{,}797$) and
$\log(1/\delta)/\KL(\mu_1\|\mu_0)\approx441$
(one-sided; leading order $2\pi\log(1/\delta)/G^2=449$); the
finite-$\delta$ lower-bound prefactors $\log(1/(2\delta))$,
$\log(1/(4\delta))$ give $1{,}331$ and $237$. The small-gap expansions converge
quickly ($8\pi\,\Ch/G^2=1.0379$ at $G{=}0.04$, $1.0001$ at $G{=}0.0025$),
and the coupling monotonicity was additionally verified by exact
trinomial enumeration (all thresholds, $m\le9$). Three consequences. (i)~Against the sharp two-sided level
$1{,}732$, the Bretagnolle--Huber floor $228$ of
Remark~\ref{rem:lecam-numbers} is loose by ${\approx}7.6\times$ (the
KL-vs-Chernoff factor $4.09$ at the anchor, times the
$\log(1/(4\delta))$-vs-$\log(1/\delta)$ deflation)---settling the
$225$-vs-$228$ proximity there as pure coincidence. (ii)~The deployed
closed-form certificate at the class cap fires from $m=634$ (one-sided;
$2{,}179$ at the half-gap): the ${\approx}1.4\times$ finite-anchor premium
over the sharp $441$ is the bracket's lower-order $M$-term and vanishes in
the corollary's small-gap limit. (iii)~The class-vs-instance distinction is
structural, not slack: the class constant is governed by
$V^\star=\pi=0.12$, the audited instance by its own
$\sigma^2_{\mathrm{routing}}=0.0384$ (ratio $3.125$). The small-gap
constants $2V^\star$/$8V^\star$ extend verbatim to the relaxed class
$\{\mu(-1)+\mu(+1)\le\pi\}$ (the lower-bound pair lies in it, and the
capped certificate is uniformly valid over it; over the unrestricted
simplex the same arguments give the constants with $V^\star=1$); the
fixed-$G$ Chernoff-exact statement is specific to fixed activity. Open:
the extremal-variance bridge $c^\star=8V^\star(\mathcal C)$ for general
convex classes, and the $\rho$-constrained (downside-budget) regime, where
the mean-zero-slice reduction can fail.
\end{remark}

\section{Leading-order sharpness of the bracket constant ($c_4=1$)}
\label{app:c4}

This section proves the main paper's \emph{leading-order sharpness of the
bracket constant} theorem (T3).

\begin{theorem}[Bernstein-tight one-sided bracket; $c_4=1$ sharp at leading
order]\label{thm:c4}
Let $Z_1,\dots,Z_m$ be i.i.d.\ with $|Z_i-\E Z|\le M$ a.s.\ and
$\Var(Z_i)=\sigma^2$, and let $\widehat G_m=\tfrac1m\sum_iZ_i$. Then for every
$\delta\in(0,1)$ and $m\ge1$,
\begin{equation}\label{eq:bern}
  \Prob\!\Big(\widehat G_m\le\E[Z]
   -\sigma\sqrt{\tfrac{2\log(1/\delta)}{m}}
   -\tfrac{2M\log(1/\delta)}{3m}\Big)\le\delta .
\end{equation}
Moreover the constant $c_4=1$ on the leading term
$\sigma\sqrt{2\log(1/\delta)/m}$ is sharp at leading order: for every $c<1$
there exist i.i.d.\ bounded sequences such that, in the iterated limit
$m\to\infty$ then $\delta\to0$, the bound
$\Prob(\widehat G_m\le\E Z-c\,\sigma\sqrt{2\log(1/\delta)/m})\le\delta$
fails.
\end{theorem}

\begin{proof}
\emph{Upper bound (Bennett--Bernstein).} By Bennett's inequality (Bennett
1962; Boucheron, Lugosi, and Massart 2013, Thm.~2.9), for any $t>0$,
\[
  \Prob\big(\widehat G_m-\E Z\le-t\big)
  \le\exp\!\Big(\!-\tfrac{m\sigma^2}{M^2}\,
      h\big(\tfrac{Mt}{\sigma^2}\big)\Big),
\]
with $h(u)=(1{+}u)\log(1{+}u)-u$. The elementary bound
$h(u)\ge u^2/(2(1+u/3))$ for $u\ge0$ (compare power series) yields the
Bernstein form
\begin{equation}\label{eq:bernraw}
  \Prob\big(\widehat G_m-\E Z\le-t\big)
  \le\exp\!\Big(\!-\frac{mt^2}{2(\sigma^2+Mt/3)}\Big).
\end{equation}
Setting the right side of \eqref{eq:bernraw} equal to $\delta$ and writing
$L:=\log(1/\delta)$ gives the quadratic
$mt^2-\tfrac{2M L}{3}t-2\sigma^2L=0$, whose positive root is
\[
  t=\frac{ML}{3m}+\sqrt{\frac{M^2L^2}{9m^2}+\frac{2\sigma^2L}{m}}
   \;\le\;\sigma\sqrt{\frac{2L}{m}}+\frac{2ML}{3m},
\]
using $\sqrt{a+b}\le\sqrt a+\sqrt b$. Since
\eqref{eq:bernraw} is monotone in $t$, the tail at the (larger) displayed
bracket is at most $\delta$, which is \eqref{eq:bern}.

\emph{Sharpness at leading order (matching Gaussian lower bound).} Take any
i.i.d.\ bounded sequence with $\sigma^2>0$ and finite third moment (e.g.\ the
audited three-point law). By the Berry--Esseen central limit theorem,
$\sqrt m(\widehat G_m-\E Z)/\sigma\Rightarrow\mathcal N(0,1)$, so for any
fixed $c\in(0,1]$,
\[
  \Prob\Big(\widehat G_m\le\E Z-c\,\sigma\sqrt{\tfrac{2L}{m}}\Big)
  \;\xrightarrow{m\to\infty}\;
  \overline\Phi\big(c\sqrt{2L}\big),
\]
where $\overline\Phi$ is the standard Gaussian upper tail. By the Mills-ratio
asymptotic $\overline\Phi(x)\sim\phi(x)/x$,
\[
  \overline\Phi\big(c\sqrt{2L}\big)
  =\delta^{\,c^2}\cdot\frac{1+o(1)}{2c\sqrt{\pi L}}
  \qquad(\delta\to0)
\]
(in this display alone, $\pi$ is the circle constant).
At $c=1$ this is $\le\delta$ for small $\delta$ (matching at leading
exponential order); for any $c<1$, $\delta^{c^2}/(2c\sqrt{\pi L})>\delta$ for
all sufficiently small $\delta$, so the claimed bound fails in the iterated
limit. Hence no constant smaller than $1$ is admissible at leading order,
while at any \emph{fixed} $\delta$ constants as small as
$z_\delta/\sqrt{2\log(1/\delta)}<1$ ($z_\delta$ the standard-normal
upper-$\delta$ quantile) (e.g.\ $0.67$ at $\delta=0.05$) remain
asymptotically admissible---the sharpness is a statement about the leading
term only.
\end{proof}

\begin{remark}[two-sided version; the bound $M$]
The two-sided bracket follows by a union bound at cost $\delta\to\delta/2$,
and $c_4=1$ remains sharp by the same CLT argument. For routing increments
$Z\in\{-1,0,1\}$ with mean $G$, $|Z-G|\le1+|G|\le2$; the main paper uses the
fixed worst case $M=2$, which depends on no estimated quantity, while the
data-dependent $1+|G|$ ($=1.04$ on the audited instance) is a valid
refinement only under a sample split. $M$ bounds the \emph{centred}
increment, which is why $M>1$ is correct even though $|Z|\le1$.
\end{remark}

\begin{corollary}[required sample size on the audited anchor]
\label{cor:mstar}
With the audited anchor estimates $\widehat G=0.040$,
$\sigma^2=\pi\sigma_E^2+\pi(1-\pi)\Delta_E^2
=0.12\cdot\tfrac29+0.12\cdot0.88\cdot\tfrac19=0.0384$
($\sigma=0.196$), fixed $M=2$, $\delta=0.05$: the smallest $m$ with
$\widehat G-B(m,\delta)>0$ is $m^\star=312$ (bisection;
Table~\ref{tab:mstar}). Comparison points at the same $\sigma^2$ (neither
involves $M$): the leading-order term alone gives
$m^\star_{\mathrm{LO}}=2\sigma^2L/G^2=143.8\!\to\!144$; Hoeffding for range $2$ gives
$3{,}745$ (uninformative at this gain scale). A direct Bennett inversion at
the same $M=2$ gives $225$ (Table~\ref{tab:bracketrob}, \S\ref{app:bracket}).
\end{corollary}

\begin{table}[!t]
\caption{The bracket $B(m,\delta)$ and certified gain
$\widehat\gamma_m=\widehat G-B(m,\delta)$ on the audited anchor
($\widehat G=0.040$, $\sigma=0.196$, fixed $M=2$, $\delta=0.05$, $\rho=0$).
$m^\star=312$ by bisection.}
\label{tab:mstar}\centering\small
\begin{tabular}{rrrrr}
\toprule
$m$ & leading & corr. & $B(m,.05)$ & $\widehat\gamma_m$\\
\midrule
$136$ & $0.0411$ & $0.0294$ & $0.0705$ & $-0.0305$ (fail)\\
$144$ & $0.0400$ & $0.0277$ & $0.0677$ & $-0.0277$ (fail)\\
$200$ & $0.0339$ & $0.0200$ & $0.0539$ & $-0.0139$ (fail)\\
$260$ & $0.0297$ & $0.0154$ & $0.0451$ & $-0.0051$ (fail)\\
$312$ & $0.0272$ & $0.0128$ & $0.0400$ & $+0.0000$ (threshold)\\
$350$ & $0.0256$ & $0.0114$ & $0.0371$ & $+0.0029$ (certify)\\
\bottomrule
\end{tabular}
\end{table}

\section{The finite-$m$ conservatism metric: where the
$\sim\!10^4\times$ figure comes from}
\label{app:conservatism}

The main paper's sharpness theorem states that at finite $m$ the full
bracket remains strictly conservative, with realized tail
${\approx}1.7{\times}10^4$ below the nominal $\delta$ at $m^\star$ on the
audited instance. This section defines the metric and shows the computation.

\begin{definition}[conservatism factor]\label{def:cons}
For a population law $\mu$ with gain $G=\E_\mu[Z]$, the conservatism factor of
the bracket at $(m,\delta)$ is the nominal-to-realized tail ratio
\[
  \kappa(m,\delta)\;:=\;
  \frac{\delta}{\Prob_\mu\big(\widehat G_m\le G-B(m,\delta)\big)} ,
\]
i.e.\ the nominal miscoverage $\delta$ divided by the realized probability of
the lower-tail event $\{\widehat G_m\le G-B(m,\delta)\}$---the Bernstein
bound's own miscoverage on this law. Validity of the bracket is
$\kappa\ge1$; large $\kappa$ quantifies slack.
\end{definition}

\begin{proposition}[the audited instance]\label{prop:800}
Let $\mu$ be the audited empirical law of the routing increment,
$\Prob(Z{=}{+}1)=0.04$, $\Prob(Z{=}0)=0.96$, $\Prob(Z{=}{-}1)=0$ (so
$G=0.04$, $\sigma^2=0.0384$, fixed $M=2$). At $m=m^\star=312$,
$\delta=0.05$: $B(312,0.05)=0.03996$, so $G-B=+0.00004$, which lies strictly
below the smallest positive value $1/312=0.00321$ attainable by
$\widehat G_m$. Hence the failure event is exactly ``no $+1$ increment in
$312$ draws'':
\[
\begin{aligned}
  \Prob_\mu\big(\widehat G_{312}\le G-B\big)
  &=\Prob\big(\mathrm{Bin}(312,0.04)=0\big)\\
  &=0.96^{312}\approx2.9\times10^{-6},
\end{aligned}
\]
and
\[
  \kappa(312,0.05)=\frac{0.05}{2.9\times10^{-6}}\approx1.7\times10^{4} .
\]
A $10^7$-path Monte Carlo (seed $20260620$) reproduces the tail as
$\approx2.8\times10^{-6}$ ($\kappa\approx1.8\times10^4$), consistent with
the closed form. This is the $\sim\!10^4\times$ figure quoted in the main
paper. For comparison at the same audited law, the direct Bennett inversion
at its own threshold $m^\star=225$ realizes tail
$1.0\times10^{-4}$ ($\kappa\approx4.9\times10^{2}$): removing the
closed-form relaxation cuts the at-threshold slack by ${\sim}35\times$; the
remaining ${\sim}5{\times}10^2$ reflects the discreteness of this extreme
no-negative-mass law, not the closed form.
\end{proposition}

\begin{table}[!t]
\caption{Exact validity computation of the one-sided bracket on the audited
law (fixed $M=2$; binomial closed form, no simulation). Rows with $G-B<0$
have realized tail exactly $0$ (since $\widehat G_m\ge0$). The tail
$\Prob(\widehat G_m\le G-B(m,\delta))$ never exceeds the nominal
$\delta=0.05$ at any $m$; conservatism $\kappa\ge1.6\times10^{2}$ throughout
the range and $\approx1.7\times10^{4}$ at $m^\star=312$
(Proposition~\ref{prop:800}).}
\label{tab:m4}\centering\footnotesize
\begin{tabular}{@{}rrrrr@{}}
\toprule
$m$ & $B(m,.05)$ & $G-B$ & 5th pct of $\widehat G_m$ & exact tail\\
\midrule
$50$   & $0.1477$ & $-0.1077$ & $+0.0000$ & $0$\\
$100$  & $0.0879$ & $-0.0479$ & $+0.0100$ & $0$\\
$200$  & $0.0539$ & $-0.0139$ & $+0.0200$ & $0$\\
$300$  & $0.0410$ & $-0.0010$ & $+0.0233$ & $0$\\
$500$  & $0.0294$ & $+0.0106$ & $+0.0260$ & $5.7\times10^{-5}$\\
$1000$ & $0.0192$ & $+0.0208$ & $+0.0300$ & $3.0\times10^{-4}$\\
\bottomrule
\end{tabular}
\end{table}

\begin{remark}[interpretation]
Table~\ref{tab:m4} confirms the bracket is a \emph{valid} one-sided
$\delta$-bound at every $m$ tested, and Proposition~\ref{prop:800} quantifies
its finite-$m$ slack at the certification threshold. Two honest qualifiers.
First, $\kappa$ is computed under the \emph{audited empirical law} (the
plug-in population); it is an instance statement, not a worst-case one.
Second, $\kappa$ decays as $m$ grows: in the CLT regime the realized tail of
the leading term approaches
$\overline\Phi(\sqrt{2\log(1/\delta)})\approx0.007$ at $\delta=0.05$, so
$\kappa$ decays toward ${\approx}7$. The $\sim\!10^4\times$ figure is therefore a finite-$m$
statement at $m=m^\star=312$---driven by the non-asymptotic correction term
and the discreteness of $\widehat G_m$---and is precisely why the main paper
pairs the sharpness theorem with the disclaimer that no finite-$m$ tightness
is claimed: ``Bernstein-tight'' means leading-order sharp, not achieved at
finite $m$.
\end{remark}

\paragraph{The frozen $m{=}50$ Bank case.}
An earlier frozen $m{=}50$ Bank sample showed an apparent $+4$\,pp gain
carried by $2$ rows; the protocol refused it
($\widehat\gamma_{50}=-0.108$). The refusal is small-sample conservatism
rather than false-positive detection: at $m{=}50\ll m^\star{=}312$ the
bracket refuses \emph{any} $+4$\,pp gain. The gain's spuriousness was
established separately, by the telemetry parsing-bug fix behind the
Bank-135 re-baseline.

\section{From the population ceiling to a learned router}
\label{app:learned}

The ceiling $\Phi\le\tfrac12 I_{\mathrm{TV}}(T)$ is a population statement; in
deployment the router is \emph{learned} from $m$ samples, and two
finite-sample consequences sharpen the guardrail. First, a greedy plug-in
router that estimates $\widehat\eta_j(t)$ and selects
$\argmax_j\widehat\eta_j(t)$ has \emph{negative} expected gain when the gate
is uninformative: with $\Phi\approx0$ every cell's apparent edge is estimation
noise, and selecting on it incurs an optimizer's-curse penalty. Second, a
\emph{threshold} router that defers to the primary unless a per-cell edge
clears a high-confidence margin abstains automatically in this regime,
recovering the primary's accuracy (gain bounded below by zero). The two
statements are formalized below.

\begin{proposition}[optimizer's curse of the greedy plug-in router]
\label{prop:curse}
Suppose the gate is uninformative in the strong sense
$I_{\mathrm{TV}}(T)=0$ (equivalently $\eta_j\equiv p_j$ a.e.\ for all $j$),
the primary is strictly marginal-best ($p_1>p_j$ for $j\ne1$), and the
complementarity event has positive mass for some competitor
($q_j:=\Prob(C_j{=}1,C_1{=}0)>0$). Let $\widehat R_m$ be the greedy plug-in
router trained on $m$ i.i.d.\ samples: on each gate cell $t$ with at least
one training sample it selects $\argmax_j\widehat\eta_j(t)$ (ties to the
primary), and defaults to the primary on empty cells. Then:
\textnormal{(a)} every router $R$ has
$G(R)=-\E\big[(p_1-p_{R(T)})\big]\le0$, with equality iff $R\equiv1$
$\mu_T$-a.e.; \textnormal{(b)} $\Prob(\widehat R_m\not\equiv1)>0$ for every
finite $m$, hence the expected deployed gain is strictly negative:
$\E\big[G(\widehat R_m)\big]<0$.
\end{proposition}

\begin{proof}
(a) With $\eta_j\equiv p_j$, $G(R)=\E[\eta_{R(T)}(T)]-p_1
=\E[p_{R(T)}]-p_1\le0$, strict whenever $R$ routes away on positive mass,
since $p_j<p_1$ for $j\ne1$. (b) Fix a competitor $j$ with $q_j>0$; among cells $t$ with
$\mu_T(t)>0$ pick one with
$q_j(t):=\Prob(C_j{=}1,C_1{=}0\mid T{=}t)\ge q_j$ (one exists, since
$\max_t q_j(t)\ge\E[q_j(T)]=q_j$). The event that the training set contains
$n_t\ge1$ samples in cell $t$, all of them with $C_j=1,C_1=0$, has
probability at least $q_j(t)^{\,n_t}\,\Prob(n_t \text{ samples in }t)>0$
for any $n_t\ge1$, and on this event $\widehat\eta_j(t)=1>0=\widehat\eta_1(t)$,
so $\widehat R_m(t)=j\ne1$. By (a), conditional on any realization with
$\widehat R_m\not\equiv1$ the deployed gain is strictly negative, and it is
never positive; taking expectations gives $\E[G(\widehat R_m)]<0$.
\end{proof}

\begin{proposition}[the threshold router abstains]\label{prop:threshold}
Under the same uninformative-gate hypothesis, let $\widehat R^{\,\tau}_m$
route away on cell $t$ only if
$\widehat\eta_j(t)-\widehat\eta_1(t)>\tau_t$ for some $j$, where the margins
$\tau_t$ are calibrated (per-cell Hoeffding or Bernstein plus a union bound
over cells and advisors) so that
$\Prob\big(\exists\,t,j:\widehat\eta_j(t)-\widehat\eta_1(t)>\tau_t\big)
\le\delta$ under $\eta_j\equiv p_j$ with $p_j\le p_1$. Then with probability
at least $1-\delta$, $\widehat R^{\,\tau}_m\equiv1$ and the deployed gain is
exactly $0$: the learned router recovers the primary.
\end{proposition}

\begin{proof}
Immediate from the calibration event: off the $\delta$-exception set no cell
clears its margin, so the router never deviates from the primary and
$G(\widehat R^{\,\tau}_m)=0$.
\end{proof}

\begin{remark}[certify-then-deploy budget]
The certify-then-deploy budget is additive:
$\Theta(\sigma^2\log(1/\delta)/G^2)$ samples to certify that a gain exists
(\S\ref{app:lecam}--\ref{app:c4}) plus $\Theta(|\mathcal T|/G^2)$ to learn
which cells to route. On an uninformative gate no finite budget yields a
positive certifiable gain, and the framework returns the safe default of
not routing. The underlying rates are standard offline policy-learning and
best-arm-identification results; we invoke them only to formalize the
guardrail rather than as new theory.
\end{remark}

\section{Calibrated positive control: full specification}
\label{app:poscontrol}

This section gives the complete generator specification, the
pre-registration record, and exact binomial confidence intervals for the
positive control reported in the main paper (the ``calibrated positive
control'' section and its calibration figure).

\subsection{Generator (pre-registered)}

A latent incident regime $r\in\{A,B,C\}$ drives two conditionally independent
advisors and a noisy gate:
\begin{center}\small
\begin{tabular}{@{}lccc@{}}
\toprule
 & $r{=}A$ & $r{=}B$ & $r{=}C$\\
\midrule
$\Prob(r)$        & $0.50$ & $0.30$ & $0.20$\\
$\Prob(C_1{=}1\mid r)$ (primary)    & $0.82$ & $0.35$ & $0.70$\\
$\Prob(C_2{=}1\mid r)$ (competitor) & $0.48$ & $0.88$ & $0.62$\\
\bottomrule
\end{tabular}
\end{center}
The gate observes the regime with accuracy $q=0.85$: $T=r$ with probability
$q$, otherwise $T$ is uniform on the other two regimes (informative, not an
oracle). $C_1\perp C_2\mid r$: diversity arises from the regime structure,
not hand-placed anti-correlation. The router is the Bayes selector
$R^\star(t)=\argmax_j\E[C_j\mid T{=}t]$, which routes to advisor $2$ exactly
on $\{T{=}B\}$ (the only gate value where the competitor is conditionally
stronger). The matched \emph{null} generator is identical except that $C_2$
is drawn from the \emph{primary's} conditional profile
($\Prob(C_2{=}1\mid r)=\Prob(C_1{=}1\mid r)$, still conditionally
independent), so the true routing gain is exactly $G=0$ while the gate and
routing rule are unchanged.

\emph{Pre-registered constants} (fixed before any bracket was computed):
seed $20260608$; $\delta=0.05$; $\rho=0$ (no shift modeled, stated
explicitly); $M=1+|G|$ (pre-registered; the reported brackets were
subsequently replaced by the conservative fixed worst case $M=2$---the
certify/withhold verdicts are unchanged); $m_{\mathrm{BIG}}=300$ and $m_{\mathrm{SMALL}}=70$,
chosen as comfortably above and below the analytically implied
$m^\star\approx100$ (analytic targets: $p_1\approx0.655$, $p_2\approx0.628$,
$\pi\approx0.31$, $\Delta_E\approx0.39$, $G\approx0.12$,
$\Var(Z)\approx0.17$; under the final fixed-$M{=}2$ bracket the realized
requirement is $m^\star=146$, and $300$/$70$ remain on the correct sides). Monte-Carlo certify-rates use $2000$ independent draws per condition
from a separate stream (seed $20260608+1$), so the frozen instance is not
reused.

\subsection{Realized frozen instance and decisions}

The frozen $m{=}300$ instance (seed $20260608$): $p_1=0.687$, $p_2=0.580$;
$\pi=0.273$, $\Delta_E=0.378$, $\widehat G=0.1033$; variance
$\widehat\sigma^2=0.1398$ (direct, ddof $1$) versus $0.1407$ via the routing
decomposition $\pi\sigma_E^2+\pi(1-\pi)\Delta_E^2$ (agreement is the internal
consistency check); the reported brackets use the fixed worst case $M=2$
(the pre-registered data-dependent value was $1{+}|\widehat G|=1.1033$).
Error diversity, measured not asserted:
contingency (both/primary-only/competitor-only/neither)
$=108/98/66/28$, oracle $0.907$, headroom $+22.0$pp, independence baseline
$+18.2$pp, excess $\mathcal E=+3.8$pp, $\phi=-0.167$; gate informativeness
$\AUC=0.752$ on discordant rows. Required sample size: $m^\star=146$
(bisection on the full bracket). Decisions on the same instance statistics:
\begin{center}
$m{=}300\ge m^\star$:\ \ $B=0.0662$,\ \ $\widehat\gamma=+0.0372$
\ $\Rightarrow$ \textbf{certify};\\
$m{=}70<m^\star$:\ \ $B=0.1664$,\ \ $\widehat\gamma=-0.0631$
\ $\Rightarrow$ \textbf{withhold}.
\end{center}

\subsection{Monte-Carlo certify-rates with exact confidence intervals}

Each condition uses $2000$ independent draws; the certify-rate is the
fraction with $\widehat\gamma_m>0$. We report exact (Clopper--Pearson) $95\%$
intervals, from the beta quantile form
$\big[\mathrm{B}^{-1}(0.025;\,k,\,n{-}k{+}1),\
\mathrm{B}^{-1}(0.975;\,k{+}1,\,n{-}k)\big]$ for $k$ successes in $n$ trials
(computed with SciPy):
\begin{center}\footnotesize
\begin{tabular}{@{}lrrc@{}}
\toprule
condition & $k/n$ & rate (\%) & $95\%$ CI (\%)\\
\midrule
positive, $m{=}300$ & $1979/2000$ & $98.95$ & $[98.40,\,99.35]$\\
positive, $m{=}70$ & $222/2000$ & $11.10$ & $[9.76,\,12.56]$\\
null ($G{=}0$), $m{=}300$ & $3/2000$ & $0.15$ & $[0.03,\,0.44]$\\
\bottomrule
\end{tabular}
\end{center}
The three rows are the calibration triptych quoted in the main paper:
\textbf{(a)} power---once $m\ge m^\star$ the protocol certifies a real,
error-diverse gain in $98.95\%$ of draws (CI lower limit $98.40\%$);
\textbf{(b)} correct withholding---below $m^\star$ it certifies the same real
gain only $11.10\%$ of the time; \textbf{(c)} type-I control---on the matched
null it certifies $0.15\%$ of draws, with CI upper limit $0.44\%$,
comfortably below the nominal $\delta=5\%$. Together with the refusals on the real
(redundant, uninformatively gated) distributions, this establishes that the
protocol is calibrated rather than merely conservative.

\section{Redundancy from a shared difficulty axis}
\label{app:redundancy}
This proves the main-paper Proposition (\emph{shared difficulty forces
redundancy}) and gives its exact numerical check.

\begin{proposition}[shared difficulty forces redundancy]
Let $C_1,\dots,C_N\in\{0,1\}$ be conditionally independent given a latent
difficulty $D$, with competences $q_j(d):=\Prob(C_j{=}1\mid D{=}d)$ each monotone
in $d$ and all in the \emph{same} direction. Then
$\mathcal{E}=\E[\max_j C_j]-\big(1-\prod_j(1-p_j)\big)\le0$ ($p_j=\E[C_j]$); for
$N{=}2$, $\mathcal{E}=-\mathrm{Cov}(C_1,C_2)\le0$ exactly.
\end{proposition}
\begin{proof}
Write $\bar C_j=1-C_j$ and $g_j(d)=1-q_j(d)=\E[\bar C_j\mid D{=}d]$; same-direction
monotonicity of the $q_j$ makes the $g_j$ monotone in $d$, all in the same
direction. Conditioning on $D$ and using conditional independence,
$\E[\prod_j \bar C_j]=\E_D[\prod_j g_j(D)]$ while $\prod_j\E[\bar C_j]=\prod_j\E_D[g_j(D)]$,
so $\mathcal{E}=\prod_j\E_D[g_j(D)]-\E_D[\prod_j g_j(D)]$. Chebyshev's association
inequality gives, for $f,h$ monotone in the same direction in a single variable
$D$, $\E[f(D)h(D)]\ge\E[f(D)]\E[h(D)]$. Inductively, $\prod_{j\le N}g_j$ equals
$g_N\cdot\prod_{j<N}g_j$ with both factors non-negative and monotone in the same
direction, so $\E[\prod_{j\le N}g_j]\ge\E[g_N]\E[\prod_{j<N}g_j]\ge\E[g_N]\prod_{j<N}\E[g_j]$.
Hence $\E_D[\prod_j g_j(D)]\ge\prod_j\E_D[g_j(D)]$ and $\mathcal{E}\le0$. For
$N{=}2$, $\E[\max(C_1,C_2)]=p_1+p_2-\E[C_1C_2]$ and $1-(1-p_1)(1-p_2)=p_1+p_2-p_1p_2$,
so $\mathcal{E}=p_1p_2-\E[C_1C_2]=-\mathrm{Cov}(C_1,C_2)$; conditional
independence gives $\mathrm{Cov}(C_1,C_2)=\mathrm{Cov}(q_1(D),q_2(D))\ge0$
for same-direction monotone $q_1,q_2$, so $\mathcal{E}\le0$.
\end{proof}

\noindent\emph{Numerical check} (exact $D$-grid integration, no Monte-Carlo
noise): the $N{=}2$ identity
$\mathcal{E}=-\mathrm{Cov}$ holds to machine precision
($\max|{\cdot}|=1.1\times10^{-16}$); across $30{,}000$ same-direction monotone
configs ($N{=}2,\dots,6$) the maximum $\mathcal{E}$ is $-5\times10^{-6}\le0$;
dropping the same-direction hypothesis (competences moving in \emph{opposite}
directions) yields $\mathcal{E}>0$ in $2000/2000$ configs (up to $+0.063$),
confirming the hypothesis is necessary. This is the structural reason the
$221$-pool screen finds $\mathcal{E}\le0$ everywhere.

\section{Bracket robustness: fixed $M$, direct Bennett, empirical Bernstein}
\label{app:bracket}
The main-paper bracket fixes $M{=}2$ (worst-case bound on the centred increment
$|Z-\E Z|\le1+|G|\le2$ for $Z\in\{-1,0,1\}$), so no constant depends on the
unknown $G$ (avoiding circularity). Table~\ref{tab:bracketrob} re-derives
$m^\star$ and $B(135)$ at the audited anchor ($G{=}0.04$, $\sigma^2{=}0.0384$,
$\delta{=}0.05$) under five \emph{valid} brackets with fixed constants: relaxed
Bernstein at $M{=}1.04$ (the old data-dependent value, reference) and $M{=}2$; the
exact Bernstein root and the direct Bennett inversion (both removing the
closed-form $\sqrt{a{+}b}\le\sqrt a+\sqrt b$ relaxation); and empirical
Bernstein~\citep{maurerpontil2009}, which assumes no $\sigma^2$ and no a priori
$M$ (range $2$).

\begin{table}[!t]
\caption{Required sample size under five valid brackets at the Bank anchor.
Fixing $M{=}2$ raises the relaxed $m^\star$ from $237$ to $312$; the direct
Bennett bound (same $M{=}2$) lowers it to $225$---within $1.3\%$ of the Le~Cam
lower bound $228$---so the inflation is the closed-form relaxation, not the fixed
constant. The assumption-light empirical-Bernstein bound is most conservative
($810$) yet still refuses Bank ($m^\star\!>\!135$). All five agree on every
reported verdict.}
\label{tab:bracketrob}\centering\footnotesize
\setlength{\tabcolsep}{4pt}
% auto-generated by bracket_robustness.py
\begin{tabular}{@{}llrr@{}}
\toprule
Bracket & constant & $m^\star$ & $B(135)$ \\
\midrule
relaxed Bernstein & $M{=}1.04$ (old) & $237$ & $0.0567$ \\
relaxed Bernstein & $M{=}2$ (headline) & $312$ & $0.0709$ \\
exact Bernstein & $M{=}2$ & $244$ & $0.0586$ \\
direct Bennett & $M{=}2$ & $225$ & $0.0543$ \\
empirical Bernstein & $R{=}2$ & $810$ & $0.1743$ \\
\midrule
Le~Cam lower bound (BH) & --- & $228$ & --- \\
\bottomrule
\end{tabular}

\end{table}

\noindent The data-dependent $M{=}1{+}|G|$ used in earlier drafts is itself valid
under a sample split (estimate $G$ on a held-out half to set $M$), at the cost of
halving the certification sample; we report the fixed worst case instead. Because
the leading $\sigma\sqrt{2\log(1/\delta)/m}$ term carries no $M$, the choice
affects only the lower-order correction, so on the large-$m$ RouterBench
certificate the $M{=}1.04\!\to\!2$ change moves $\widehat\gamma^{\mathrm{iid}}$ by
${\approx}1{\times}10^{-4}$ (verdict unchanged); the plug-in $\widehat\sigma^2$ is the only
remaining estimated quantity, and the empirical-Bernstein column bounds the
resulting uncertainty without assuming $\sigma^2$ known.

\section{Numerical verification summary}
\label{app:veriftable}
Every theorem and proposition in the paper carries an independent numerical
check against frozen golden files. Table~\ref{tab:verif} consolidates all
checks in one place; the remarks in the sections above retain the per-check
details.

\begin{table}[!t]
\caption{Consolidated numerical verification of every theoretical claim.}
\label{tab:verif}\centering\footnotesize
\renewcommand{\arraystretch}{1.15}
\begin{tabular}{@{}p{0.30\columnwidth}p{0.30\columnwidth}p{0.27\columnwidth}@{}}
\toprule
\textbf{Claim} & \textbf{Method} & \textbf{Outcome}\\
\midrule
T1(a): $\max_RG=\Phi$ &
brute force, $200$ random laws ($N{=}2$, $K{=}2$), all $512$ routers &
$200/200$ matches; max dev.\ $2.9{\times}10^{-16}$\\
T1(b): AUC-insufficiency witnesses; mixture family $\Phi=\alpha/4$ &
direct computation at $\alpha\in\{0,\frac14,\frac12,\frac34,1\}$ &
machine precision (\S\ref{app:T1})\\
Ceiling $\Phi\le\frac12I_{\mathrm{TV}}$ &
$2{\times}10^5$ random laws ($N{=}3$, $|\mathcal T|{=}4$) &
$0$ violations; sup ratio $0.9956$; equality witness exact\\
T2 witnesses A/B; no-go pair C/D; (C1) control &
independent implementation &
$G=0.03$, $0.10$; $0$ vs $0.06$ at profile $(1,1,1)$; control $0$
(\S\ref{app:T2})\\
Redundancy prop.\ ($\mathcal E\le0$) &
exact $D$-grid, $30{,}000$ monotone configs, $N{=}2..6$ &
max $\mathcal E=-5{\times}10^{-6}$; $N{=}2$ identity to $1.1{\times}10^{-16}$; opposite-direction control $\mathcal E>0$ in $2000/2000$\\
Coupling monotonicity (\S\ref{app:sharpconst}) &
exact trinomial enumeration, all thresholds, $m\le9$ &
no violations\\
Sharp-constant anchors ($\Ch$, $\KL$, $m^\star$, cap certificate) &
independent recomputation vs.\ golden &
all values reproduce (Remark~\ref{rem:sharp-numbers})\\
Bracket variants (relaxed $M{=}2$ / direct Bennett / empirical Bernstein) &
recomputation on all reported instances &
every verdict agrees (\S\ref{app:bracket})\\
\bottomrule
\end{tabular}
\end{table}

\section*{Artifact index}
Every number in this appendix regenerates from the frozen, content-hashed
artifacts in the artifact package released with the published version:
end-to-end scripts for the
ceiling inequality and equality witnesses, the T1 sweep and mixture family, the T2
witnesses and no-go pair (with golden output), the bracket-validity simulation and
$m^\star$ table, the semi-synthetic
positive control (with golden output), the independence baseline, the five-bracket
robustness table, the oracle-vs-feasible gate ladder on real Bank, and the exact
shared-difficulty check; a single source-of-truth constants module supplies $G$,
$\pi$, $\Delta_E$, both variance objects, $M$, the brackets, and every $m^\star$;
the deterministic router invokes no language model.

\end{document}